\documentclass[two side]{article}

\usepackage[accepted]{aistats2025}
\usepackage{graphicx}
\usepackage{subcaption}

\usepackage[utf8]{inputenc}
\usepackage{bbold}
\usepackage[english]{babel}
\usepackage[authoryear]{natbib}
\setcitestyle{round}
\usepackage{csquotes}
\usepackage[T1]{fontenc}
\usepackage{amsthm}
\usepackage{amsmath}
\usepackage{amsfonts}
\usepackage{hyperref}
\usepackage{bookmark}
\usepackage{authblk}
\usepackage{graphicx}
\usepackage{multirow}
\hypersetup{
    colorlinks=true,
    linkcolor=cyan,
    citecolor=blue,    
    urlcolor=red,
    pdfpagemode=FullScreen,
    }
\usepackage{graphicx}
\usepackage{array}

\newtheorem{theorem}{Theorem}
\newtheorem{lemma}{Lemma}
\newtheorem{prop}{Proposition}
\newtheorem{coro}{Corollary}
\newtheorem{defi}{Definition}
\newtheorem{assumption}{Assumption}

\begin{document}

\twocolumn[

\date{}

\aistatstitle{Memorization in Attention-only Transformers}

\aistatsauthor{Léo Dana, Muni Sreenivas Pydi, Yann Chevaleyre}

\aistatsaddress{LAMSADE, Paris-Dauphine University.}]

\begin{abstract}
    Recent research has explored the memorization capacity of multi-head attention, but these findings are constrained by unrealistic limitations on the context size. We present a novel proof for language-based Transformers that extends the current hypothesis to any context size. Our approach improves upon the state-of-the-art by achieving more effective exact memorization with an attention layer, while also introducing the concept of approximate memorization of distributions. Through experimental validation, we demonstrate that our proposed bounds more accurately reflect the true memorization capacity of language models, and provide a precise comparison with prior work.
\end{abstract}

\section{INTRODUCTION}
\label{section:intro}

Modern large language models, especially Transformers, showcase great memorization capacity \citep{karpukhin-etal-2020-dense, roberts_how_2020}. Among recent works, researchers have shown that facts are memorized in the MLPs of a Transformer, and have even identified which MLPs \citep{meng2022locating,meng2023massediting,nanda_fact_2023}. However, they were not able to understand \textit{how} these MLPs store information. Both exact and approximate theoretical memorization in an MLP are well documented in the literature: a ReLU MLP can memorize exactly as many real-valued label as it has neurons $n$ \citep{Bubeck2020NetworkSA}, and can memorize exactly $n$ discrete labels with only $\Tilde{O}(\sqrt{n})$ neurons \citep{vardi2022on}.

Contrary to the MLPs, the memorization power of multi-head attention layers has not been empirically studied. The main role of the attention layer is not viewed as remembering information but rather as moving between residual streams the information retrieved by the MLPs \citep{nanda_fact_2023,wang2023interpretability, variengien2024look}. For the theoretical aspect of the memorization in attention layers, there exists results on the expressivity of the attention patterns \citep{pmlr-v119-bhojanapalli20a}, the memorization capacity of attention layers \citep{mahdavi2024memorization}, and the memorization capacity of Transformers \citep{kim2023provable}. We will discuss related works in depth in section \ref{section:related_work}.

In this article, we are interested in moving the state-of-the-art in terms of memorization capacity for the attention layer. We will thus consider the memorization capacity of an Attention-only Transformer (AoT). We need to specify what \textit{memorization capacity} means. We will distinguish two types of memorization tasks, namely the association task and the distribution task.

The \textbf{association} task, already studied by \citet{bietti2023birth,cabannes2024scaling, kim2023provable,mahdavi2024memorization}, consists of predicting a token given a sequence of tokens as input. We only require the AoT to predict the right next-token at the last position. This memorization is exact and hence, we want to know the maximal set of sequence-token associations $(t_{in},t_{out})$ that can be exactly memorized by an AoT.

The \textbf{distribution} task consists of predicting the correct distribution, measured using the $KL$-divergence, for an input sequence of tokens. We use the $KL$-divergence since it is the default loss function used to train most Transformers. To our knowledge, we are the first to introduce and study this task. Memorizing distribution happens in natural language modeling: take the sentence "Arnold Schwarzenegger was a", it can be completed with "actor", "writer" or "bodybuilder". Thus, language models need to memorize not one but several correct next-tokens, each with a possibly different probability depending on the importance of the answer.

Our contributions are:
\begin{enumerate}
    \item We improve the state-of-the-art on the association task by proving that a one layer AoT with $H$ heads each of dimension $d_h$, and an embedding dimension $d$ can memorize $Hd_h+d$ associations. In the context of language model, this improves on the previous result on the attention layer expressivity by \citet{mahdavi2024memorization} which requires a limited context-windows and has memorization capacity of $H(d_h-1)+1$. We compare our result with other constructions using deep Transformers \citep{kim2023provable} as well as MLPs memorization \citep{huben2023attentiononlytransformersimplementingmlps,Bubeck2020NetworkSA}.
    \item We introduce the distribution task for Transformers as a way to quantify memorization when there is not a unique correct next-token. We provide upper and lower bounds for that distribution task on the error made by the best AoT. The divergence of that AoT will approximate that of a sequence encoder, which is a mapping from token sequences to logits, and has a rank constraint.
    \item Finally, we prove upper bound for the divergence of sequence encoders when the target distribution is almost a look-up table.
\end{enumerate}
Proofs for the statements can be found in appendix with experimental details. The code-base is available at \href{https://github.com/leodana2000/Transformer_Attentional_Memory}{this link} .

\section{FORMALISM}
\label{section:formalism}
We study an Attention-only Transformer (AoT) that has only one layer of multi-head attention (MHA) mechanism denoted by $\mathcal{A}$. Let $[N]=\{1, ..., N\}$ be the token dictionary. Each token is embedded in dimension $d$ by the embedding $e:[N]\rightarrow\mathbb{R}^{d}$, and a positional embedding is added based on its position $s$. Thus, we denote a sequence of $S$ tokens by $t_{1:S}$. We also denote $t_{S+1}$ for the output token. The MHA contains $H$ heads, each of inner dimension $d_h$, with $d_h\leq d$\footnote{Since matrices $W_O$, $W_V$ and $W_Q$, $W_V$ are multiplied together, taking $d_h\leq d$ will decrease their rank to $d_h$, but taking $d_h>d$ will simply have them full rank, which is suboptimal compared to taking $d=d_h$.}, meaning that $W_{Q}^h, W_{K}^h, W_{V}^h \in\mathbb{R}^{d_h, d}$ and $W_O^h\in\mathbb{R}^{d,d_h}$. Following the intuition from \citet{elhage2021mathematical}, we choose to separate output matrices in each head, as well as combine $W_{QK}^h = (W_Q^h)^TW_K^h$. Each attention head can be written as follows.
\begin{equation}
    \begin{split}
        \mathcal{A}^h(t_{1:S}) &= W_O^{h}W_V^hA^h(t_{1:S})\\
        &= W_O^{h}W_V^h\sum_{s=1}^Sa^{h}(t_{1:S})_s(e(t_s)+pos_s)\\
        a^{h}(t_{1:S}) &= \text{Softmax}(\text{raw}(t_{1:S})_s, s=1:S)\\
        \text{raw}(t_{1:S})_s &= (e(t_S)+pos_S)^TW_{QK}^h(e(t_s)+pos_s)
    \end{split}
\end{equation}
We concatenate the output matrices into $W_O\in\mathbb{R}^{d, Hd_h}$ and the attention before output $A(t_{1:S})\in\mathbb{R}^{Hd}$. We also construct the matrix $W_V\in\mathbb{R}^{Hd_h,Hd}$ as block diagonal with block $W_V^h$, which has full rank when each $W_V^h$ has full rank. Then, the output of the attention layer is added to the residual stream and goes through an unembedding matrix $W_U\in\mathbb{R}^{N,d}$ to obtain logits for the next token. Since we are only interested in the next-token prediction at the last position, we denote the AoT's computation by 
\begin{equation}
    \begin{split}
    \label{eq:AoT}
        &\mathcal{T}(t_{1:S}) = W_U\left(e(t_S)+pos_S+W_OW_VA(t_{1:S})\right)
    \end{split}
\end{equation}
For both the association and distribution tasks, we define the conditional distribution $\pi_{t_{1:S}}$ over next token $t_{S+1}$ and a prior distribution $\pi$ over token sequences $t_{1:S}$. The task of the Transformer is to minimize the $KL$-divergence with the conditional distribution for each input sequence, averaged over the prior distribution. \[d_{KL}(\pi, \mathcal{T}):=\mathbb{E}_{t_{1:S}\sim\pi}[d_{KL}(\pi_{t_{1:S}}||\text{Softmax}\circ\mathcal{T}(t_{1:S}))]\] The association case is a restriction of the distribution case to conditional distributions with 0 entropy, which is equivalent to having one next-token of probability 1. We denote this setting as assumption 1 below, and we will use it when referring to the association task.
\begin{assumption}
    \label{assump:1}
    For all sequence token $t_{1:S}$, there exists $t_{S+1}$ such that $\pi(t_{S+1}|t_{1:S})=1$. This is equivalent to $\pi$ having conditional distributions with 0 entropy.
\end{assumption}
In the association case, we say that the Transformer memorizes an example $(t_{1:S}, t_{S+1})$ if $\mathcal{T}(t_{1:S})_{t_{S+1}}$ is the maximum logit, and we let $T_0$ the number of sequence-token association to memorize, which is at most $N^S$. In the distribution case, the Transformer memorizes example $t_{1:S}$ if $\mathcal{T}(t_{1:S}) = \log(\pi_{t_{1:S}})$. We introduce another assumption that arises in the distribution case.
\begin{assumption}
    \label{assump:2}
    For all $t_{1:S}, t_{S+1}$, $\pi(t_{S+1}|t_{1:S})\neq 0$. This is equivalent to $\pi$ having conditional distributions with full support.
\end{assumption}
%
\section{THE MEMORIZATION LIMIT OF TRANSFORMERS}
\label{section:memorization_limit}

Looking at equation (\ref{eq:AoT}), we see that the AoT can be written as $\mathcal{T}(t_{1:S}) = W_UE(t_{1:S})$ where $E(t_{1:S})=e(t_S)+pos_S+W_OW_VA(t_{1:S})$ is a sequence embedding. Our AoT belongs to the set of \textbf{sequence encoders} defined below.

\begin{defi}
    The set of maps that embed token sequence and unembed them into logits is denoted \[\mathcal{L}(N,S,d) := \left\{f_{W,E}|W\in\mathbb{R}^{d,N}, E:[N]^S\rightarrow\mathbb{R}^{d}\right\}\] with $f_{W,E}(t_{1:S})=WE(t_{1:S})$. We call them sequence encoders and we define \[d_{KL}\left(\pi, \mathcal{L}(N,S,d)\right):=\underset{f\in\mathcal{L}(N,S,d)}{\inf}d_{KL}(\pi, f).\]
\end{defi}

In full generality, Transformers with any number of MLPs or attention layers are sequence encoders. Indeed, one can think about every computation happening before the unembedding as a sequence embedding parametrized by few parameters (in comparison to unconstrained sequence embdedding). Thus, as stated in Proposition \ref{prop:lower_bound} below, Transformers can memorize distribution at most as well as the best sequence encoders.

\begin{prop}
    \label{prop:lower_bound}
    Let $\mathcal{T}$ be any Transformer with embedding dimension $d$, dictionary size $N$ and context window $S$, and $\pi$ be any distribution, we have 
    \begin{equation}
        d_{KL}(\pi, \mathcal{T}) \geq d_{KL}\left(\pi, \mathcal{L}(N,S,d)\right).
    \end{equation}
\end{prop}

In particular, as stated in Proposition \ref{prop:no_bottleneck}, when $d < N-1$, the infimum can be non-zero. This creates a rank $d$ bottleneck for the memorization of distributions by the Transformer. This means that in general one cannot have a Transformer remember more than $d$ distributions exactly. As we will see in Theorem \ref{thm:main}, approximate memorization is more suitable to the distribution task, as measured by the distance to this lower-bound.

\begin{prop}
    For any distribution $\pi$, if $d\geq N-1$, then $d_{KL}\left(\pi, \mathcal{L}(N,S,d)\right)=0$. Conversely, if $d<N-1$, there exists a distribution $\pi$ such that $d_{KL}\left(\pi, \mathcal{L}(N,S,d)\right)>0$.
\end{prop}

For the association task, the bottleneck doesn't exist in any embedding dimension $d\geq 2$. The appropriate way to evaluate memorization in the association task is using exact memorization.

\begin{prop}
    \label{prop:no_bottleneck}
    Under Assumption 1, for $d\geq 2$, $d_{KL}\left(\pi, \mathcal{L}(N,S,d)\right)=0$.
\end{prop}

While sequence encoders don't have limit to their associative memorization capacity, under-parametrized AoT do. We give an upper bound on the memorization capacity of a  one-layer AoT in Corollary \ref{corollary:1}, and  we compare it to experimental scaling laws in section \ref{section:experiments}.

\section{MEMORIZATION CAPACITY OF ATTENTION-ONLY TRANSFORMERS}
\label{section:memorization_capacity}

In this section, we present our main results that respectively give upper bounds on memorization in the distribution and association settings. We will start with the result on remembering distributions, making the association task a corollary.

Define $T_{\varepsilon}$ as the smallest number of token sequences whose cumulative probability is greater than $1-\varepsilon$. We have $T_{\varepsilon}\leq \left\lceil (1-\varepsilon)N^S\right\rceil$, the upper bound being attained when the probability distribution over token sequences is uniform. This notation is consistent with $T_0$ defined earlier, the number of non-zero probability sentences. The theorem below states that we can construct a Transformer which approximates the lower bound set in Proposition \ref{prop:lower_bound} arbitrarily.

\begin{theorem}
    \label{thm:main}
    Let $\varepsilon \geq 0$ and $\gamma >0$. Under Assumption 2 there exists $f_{W,E}$ and an AoT $\mathcal{T}$ with embedding dimension $d$, head dimension $d_h$, and $H$ attention heads, satisfying $d_hH+d \geq T_{\epsilon}$, such that 
    \begin{equation}
        d_{KL}(\pi, \mathcal{T}) \leq d_{KL}(\pi, \mathcal{L}(N,S,d)) + C\varepsilon||WE||_2 +\gamma
    \end{equation}
    $\mathcal{T}$ has $d(S+2N+4d_hH)$ parameters.
\end{theorem}

\textbf{Remarks 1.}
\begin{itemize}
    \item In the parameter count, $dS$ and $dN$ are necessary for the word embedding, positional embedding and unembedding. The $dd_hH \geq dT_{\varepsilon}-d^2$ scaling comes from the attention heads and is the comparison point with previous work as we will detail in section \ref{section:related_work}.
    \item In practice, sequence-token pairs don't have the same probability. Thus, our AoT remembers all most likely sequences up to $\varepsilon$, which explains that we are close to the lower bound for small epsilon. The term $||WE||_2$ corresponds to the worst possible prediction, and is standard in the literature on MLP expressivity\footnote{See Theorem 2 by \citet{Bubeck2020NetworkSA}.}. The term in $\gamma$ is simply a term that can be taken infinitely close to 0, and accounts for approximating the skip connection with a special attention head.
    \item The constant $C \geq 1$ depends on the matrix $W_VA$ and is finite thanks to its high rank. In order to bound $C$ effectively, one has to understand more precisely the singular value decomposition of $W_VA$. See equation (14) of the appendix.
    \item One can take $\varepsilon=0$ to achieve the lower bound set in Proposition 1 with $Hd_h+d = T_{0}$ attention heads, and uses $d(S+2N+4(T_0-d))$ parameters. In that case, one could say that the Transformer remembers exactly the distributions, even if the divergence is not 0, since it is the smallest loss attainable by the Transformer architecture.
\end{itemize}
We can also use our Theorem \ref{thm:main} under Assumption \ref{assump:1}: we use it on the probability $\pi_{\delta} = \frac{\pi + \delta}{1+N\delta}$ with $\delta >0$ small enough to satisfy Assumption \ref{assump:2}, such that the Transformer has perfect accuracy for $\pi$. In this association task, the AoT thus remembers $T_0$ associations.

\begin{coro}
    \label{corollary:1}
    Under Assumption 1, there exist an AoT $\mathcal{T}$ with embedding dimension 2, $H$ attention heads and $d_h$ inner head dimension, which can memorize at least $T_0 = Hd_h+2$ associations, using $2(S+2N+4(T_0-2))$ parameters.
\end{coro}

\subsection{Sketch of the Proof}
We present below the proof of our main result. First, we prove the case $\varepsilon=0$ as it is useful for the $\varepsilon>0$ case. Assumption \ref{assump:2} is used in both cases for Lemma \ref{lemma:1} below.

\begin{lemma}
    \label{lemma:1}
    Under Assumption 2, there exists $f\in\mathcal{L}(N,S,d)$ such that $d_{KL}(\pi, f) = d_{KL}(\pi,\mathcal{L}(N,S,d))$
\end{lemma}

\textit{1. Case $\varepsilon=0$}: we have our AoT as in equation (2), and thanks to Lemma 1 we take $f_{W, E}$ an optimal sequence encoder that we approximate. By choosing $W_U = W$, we want to solve the system $E(t_{1:S}) = e(t_S)+pos_S+W_OW_VA(t_{1:S})$ for all token sequence $t_{1:S}$.

We start by showing that the skip connection $e(t_S)+pos_S$ can almost be written as an attention head, by a change of basis \[e(t_S)+pos_S = W_O^0W_V^0(e'(t_S)+pos'_S)\] and using the attention pattern $W_{QK}^0 = \lambda I_d$. When $\lambda$ is large enough, the attention head with matrices $W_{QK}^0, W_V^0, W_O^0$ becomes arbitrarily close to $e(t_S)+pos_S$. Thus, we can augment $A'$, the attention before output, by adding head 0, which has an inner dimension of $d$. Since the result is only valid in the limit of $\lambda\rightarrow+\infty$, this explains the constant $\gamma>0$ in the theorem.

Now, $W_VA'\in\mathbb{R}^{Hd_h+d, T_0}$, and since $T_0 \leq Hd_h+d$, the linear system $E(t_{1:S})= W_OW_VA'(t_{1:S})$ is solvable, where the variable is $W_O$, if the family $\{W_VA'(t_{1:S})\}_t$ has rank $T_0$, which is equivalent to having $\{A'(t_{1:S})\}_t$ with rank $T_0$ since $W_V$ reduces the rank. We are left with proving Lemma \ref{lemma:2}.

\begin{lemma}
    \label{lemma:2}
    Let $A$ the attention before output of a model with $H$ heads. For $T_0$ token sequences, there exists matrices $W_{QK}^i$ and embeddings $e$ and $pos$ such that the family $\{A(t_{1:S})\}_t$ has rank greater than $\min\left(T_0, Hd\right)$.
\end{lemma}

First, one can see that it is sufficient to prove the Lemma when $T_0\geq Hd$, since otherwise one can add dummy examples to have $T_0 = Hd$. We compute the rank in terms of the rows, head by head, meaning that Lemma 2 reformulates as proving that the function $W_{QK}^h \rightarrow A^h$ has its image not contained in any hyperplane of $\mathbb{R}^{d, T_0}$.

This is now a technical result: recall that an attention head output is a weighted average of the tokens by the attention pattern. Now, suppose that this function is orthogonal to some vector $v\in\mathbb{R}^{d, T_0}$ for all matrices $W$. Using Taylor expansions, we transform the softmax from the attention pattern into infinite sums of exponential. The equation is now stating that exponential functions sum to 0 for all $W$. Thus, by recognizing that exponential functions form a free family, we identify the coefficients for each different exponent. Using handcrafted exponent, we show that the coefficients before the exponents form a system, the only solution of which is $v=0$. 

\textbf{Remark 2.} In the proof of Lemma \ref{lemma:2}, one doesn't need to have $W_{QK}^i$ with a rank more than 1. This means that some parameters could be saved, using this construction with $W_Q,W_K\in\mathbb{R}^{d}$, but still $W_V\in\mathbb{R}^{d_h,d},W_O\in\mathbb{R}^{d,d_h}$. The number of parameter needed in Theorem 1 is thus $d(S+2N+2(d_h+1)H)$. 

\textit{2. Case $\varepsilon>0$}: here, we reuse the previous idea on the most likely sequences that cumulate $1-\varepsilon$ of the probability. So we solve the same linear system as before for the most likely sequences, and we control the error for the low probability sequences. If we note $S_2$ the set of the low probability sequences, we can bound the difference in $KL$-divergence by the $L_2$ norm of the logits 

\begin{equation*}
    \begin{split}
        |d_{KL}(\pi, f_{W, E})& - d_{KL}(\pi, \mathcal{T}^*)| \\
        &\leq \mathbb{E}_{t_{1:S}\in S_2}\left[|WE(t_{1:S}) - \mathcal{T}(t_{1:S})|\right]\\
        &\leq \varepsilon||WE - \mathcal{T}||_2
    \end{split}
\end{equation*}

With the same choice as in the first case, we can bound the $L_2$ norm of the logits difference by the term $C||WE||_2$ which represent the error of the best predictor, and the constant in the theorem is $C = \sqrt{1+||(P_1VP_1^T)^{-1}||_2^2}$ where $P_1$ is the projection that keeps the most likely sentences, and $W_VA' = UI_{\Sigma}P_1V$ the singular value decomposition of $W_VA'$\footnote{Here, the matrices $V$ and $W_V$ denote different object and are not subscript from one another. We choose to keep the usual notations from the singular value decomposition as well as the attention mechanism.}. The constant $C$ is the term that depends on the expressivity of the attention mechanism itself and is thus hard to control with only Lemma \ref{lemma:2}.

A natural question is whether one can drop Assumption \ref{assump:2} ? To do so, one can take a sequence encoder $f_{\delta}$ which is approximating the minimum divergence at $\delta$, and whose limit distribution is $\pi_0$. Then, using the same ideas as above, we approximate $f_{\delta}$ using the AoT. Thus, one has to control the growth of $||f_{\delta}||_2$. But the assumption on $f_{\delta}$ ensures that the convergence is fast enough, but we need an upper bound on the speed of convergence. Thus, it is sufficient to take $c>0$ and $\alpha\geq 1$ such that for any token sequences $t$ satisfying $\pi_0(t_{S+1}|t_{1:S})>0$, we have the property

\begin{equation}
        c\delta^{\alpha} \leq \left|\log\left(\frac{\pi_0(t_{S+1}|t_{1:S})}{\text{softmax}(f_{\delta}(t_{1:S}))_{t_{S+1}}}\right)\right| \leq \delta
\end{equation}

The lower bound will be used to prove that the norm of $f_{\delta}$ can be controlled by $O\left(\log\left(\frac{1}{\delta}\right)\right)$, while the upper bound gives the convergence of $f_{\delta}$ to the desired distribution. We develop this idea in Appendix \ref{section:poly_approx} to prove a weaker alternative to Theorem \ref{thm:main} which holds for any distribution $\pi$.

\section{EXPERIMENTS}
\label{section:experiments}
In Corollary \ref{corollary:1}, we obtained a lower bound on the number of associations a one layer AoT can remember. We now want to know what the empirical scaling of the memorization power is. To this end, we will analyze scaling laws of the accuracy of trained AoT (experiments 1 to 4) as well as MLP-based Transformers (experiment 5) for different values of $H$, $d$, and $d_h$.

\subsection{AoT Memorization}
\label{section:AoT}

\begin{figure}[t]
    \begin{subfigure}{\linewidth}
        \centering
        \includegraphics[width=\linewidth]{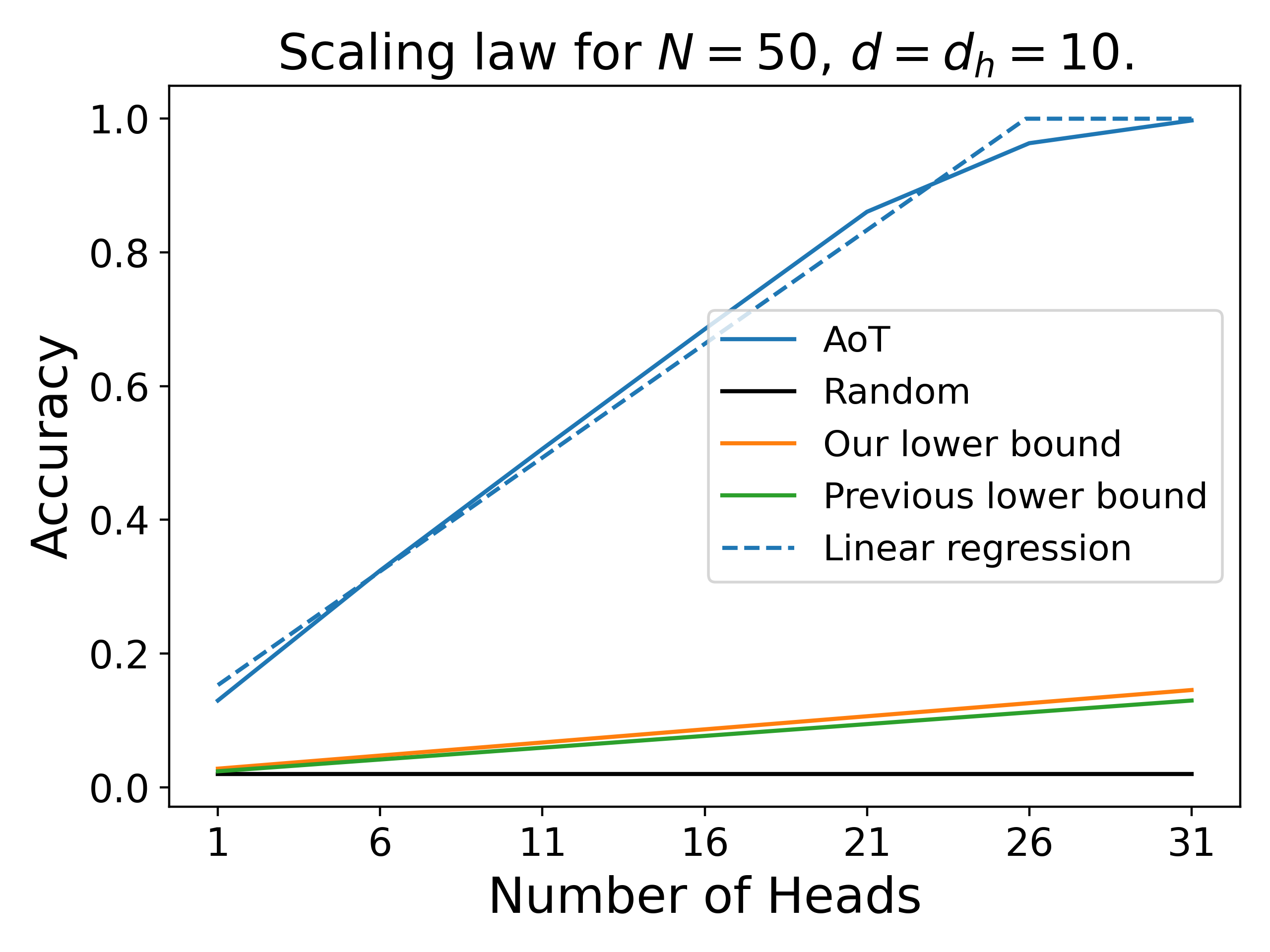}
        \caption{Experiment 1. Accuracy scaling law as $H$ grows. The head dimension is fixed to 10. The scaling is linear.}
    \end{subfigure}
    \hfill
    \begin{subfigure}{\linewidth}
        \centering
        \includegraphics[width=\linewidth]{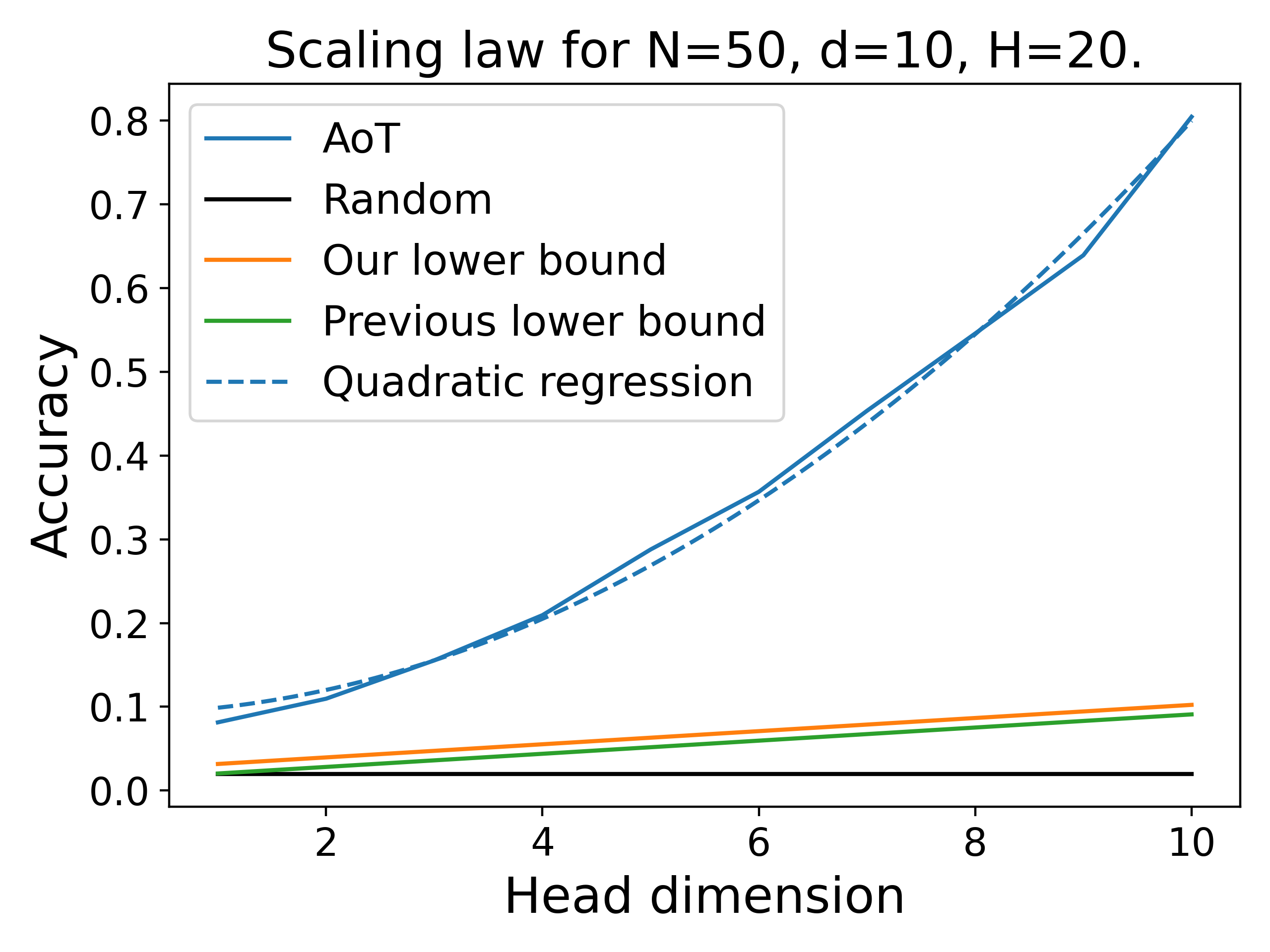}
        \caption{Experiment 2. Accuracy scaling law as $d_h$ grows. The model has 20 heads. The scaling is quadratic.}
    \end{subfigure}
    \caption{Scaling laws on $H$ and $d_h$. The embedding dimension is 10 and the dictionary size is 50. The blue dotted lines are the linear or quadratic least square approximation of the empirical accuracy.}
\end{figure}

In our experiments, we train an AoT on bigrams, meaning that $S=2$, and with a dictionary size of $N=50$. Since we want to measure associative memory, we take $\pi$ satisfying Assumption \ref{assump:1}. Moreover, we take the prior distribution of $\pi$ uniform over all pairs of tokens. This way, the accuracy will be a good measure of the associative memory. But accuracy measures the performance of the AoT in average, meaning that by pure chance, using 0 head, the AoT will have $\frac{1}{N}$ accuracy, and our Corollary \ref{corollary:1} is a statement in worst-case. Thus, to compare the scaling laws in accuracy with our results, we first need to state Corollary \ref{corollary:2}.

\begin{coro}
    Under Assumption \ref{assump:1} and for a uniform prior, there exist an AoT $\mathcal{T}$ such that \[\mathbb{E}_{t_{1:S}}[Acc(\mathcal{T}(t_{1:S}))]\geq \frac{1}{N}+\left(1-\frac{1}{N}\right)\frac{Hd_h+d}{T_0}.\]
    \label{corollary:2}
\end{coro}

In general, the transformation \[\phi:X\in[0,T_0]\rightarrow \frac{1}{N}+\left(1-\frac{1}{N}\right)\frac{X}{T_0}\in[0,1]\] gives the correspondence between associative memory and accuracy. We will use Corollary \ref{corollary:2} as a lower bound for our experiment, which we call \textit{Our lower bound}. The other lower bound called \textit{Previous lower bound} is the one obtained by \citet{mahdavi2024memorization}, which we improve on, and discuss in section \ref{section:related_work}. For more details on the experiments, please read appendix \ref{appendix:experiments}.

As said before, we want to understand the empirical scaling laws of the associative memory. In particular, we are interested in the following questions: is the scaling linear in $H$ and $d_h$ ? What is the effect of $d$ on the scaling laws ? To answer the first question, we trained AoT with $d=10$ fixed, we vary $H$ and $d_h$, and we measure the accuracy after training. Then, we plot the scaling laws when $d$ grows with a constant $d_h$. Both lower bound for the memorization power of the AoT are displayed, along with the baseline accuracy of a random AoT.

As observed on figure 1.a, the scaling in $H$ is linear for $d_h$ fixed, and on 1.b the scaling in $d_h$ is quadratic when $H$ is fixed. This means that the memorization power empirically scales as $C(d,N)Hd_h^2+C'(d,N)$. Although the scale is linear in $H$ as in Corollary \ref{corollary:1}, the constant is not the same as the growth is actually much faster. The comparison cannot be tight since Corollary 1 gives the same scaling in $H$ for all $d\geq 2$\footnote{This is all the more true when taking into account the remark from the last section on the rank of $W_{QK}$.}. In Appendix \ref{appendix:experiments}, we test the scaling of $H$ with $d=2$, showing in this case that the empirical scaling law is still faster than our bound.

The constant term $C'(d,N)$ is proven to be linear in $d$ by \citet{cabannes2024scaling}, meaning that a matrix multiplication of rank $d$ can store at most $d$ associations. Interestingly, this experiment shows that taking $d=d_h$ is optimal in terms of memorization per parameter, since the parameter scaling is linear in $d_h$ yet quadratic (for $d_h\leq d$) in memorization power. This quadratic scaling seems to be linked to the increase in expressivity of the attention pattern rather than the increase in the rank. Indeed, the rank increase when the number of heads increases only leads to a linear increase in the memorization capacity.

Next, we try to understand the scaling of $C(d,N)$ in the variable $d$. First, we fix $H=20$ and $d_h=10$ to plot the scaling law. As shown on figure 2.a, the scaling law is composed of two parts: $d\leq d_h$ and $d\geq d_h$. When $d\geq d_h$, the scaling is expected to be constant because the increase in embedding space is not beneficial to the heads when they have smaller inner dimension. The fact that we observe a greater linear scaling for $d\leq d_h$ also suggest that $C(d,N)$ is linear in $d$ for $d\leq d_h$ only.

Second, we compute the scaling of $C(d,N)$ using $d=d_h$. This way, we can compute $C(d,N)d_h^2$ by measuring the growth of the accuracy with $H$ for different values of $d$. Figure 2.b shows the growth of this scaling coefficient and compares it to different models. Each model is computed to be a least-square approximation. We observe that the coefficient's growth is between quadratic and cubic when $d=d_h$. 

Combining experiments 1,2 and 3 means we should expect a cubic scaling from experiment 4. However, we provide further experiments in larger dimension and with more layers, in Appendix \ref{appendix:experiments}, suggesting that the scaling might be only quadratic.

\begin{figure}[t]
    \begin{subfigure}{\linewidth}
        \centering
        \includegraphics[width=\linewidth]{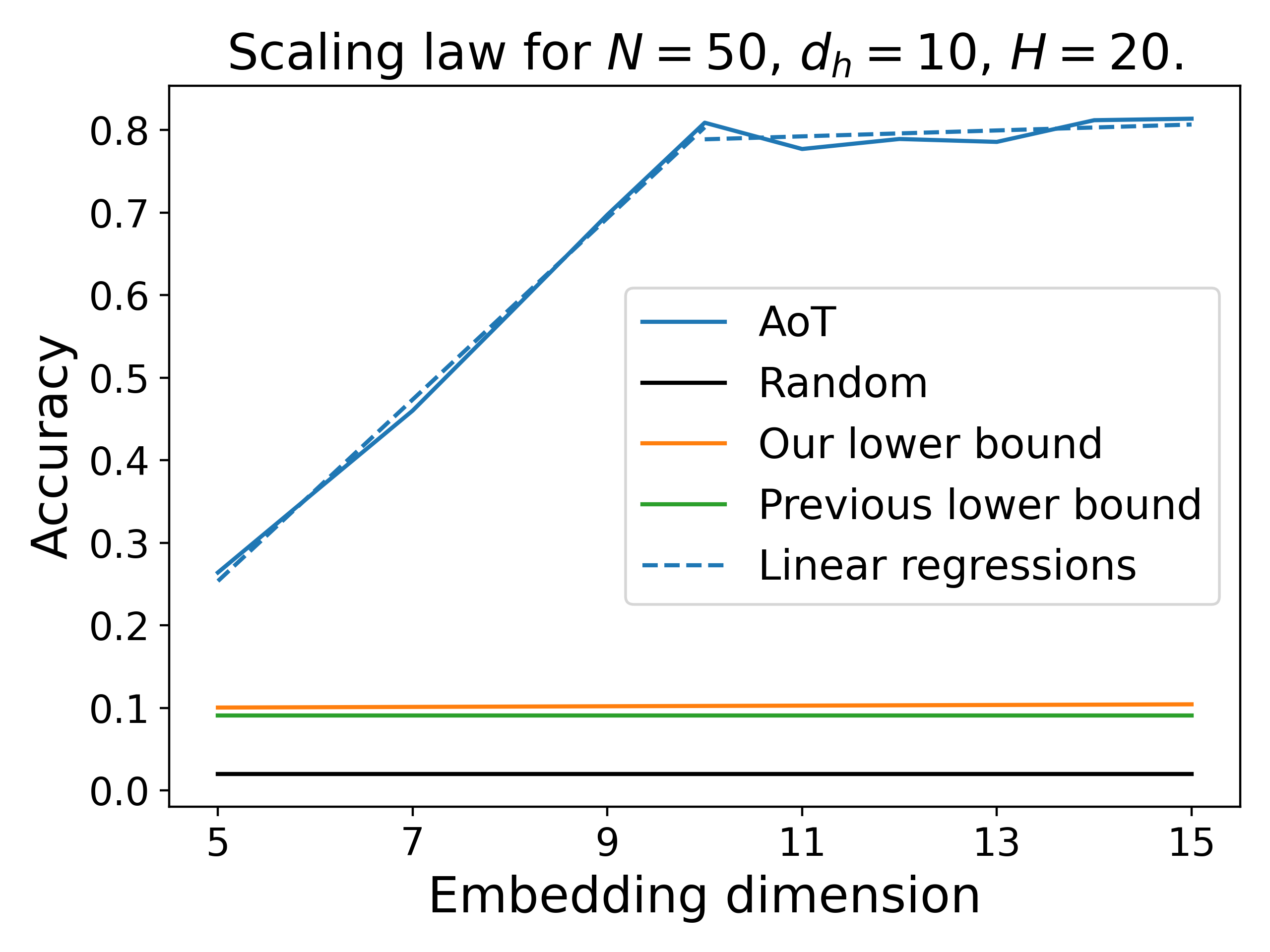}
        \caption{Experiment 3. Accuracy scaling law as $d$ grows. The head dimension is fixed to 10 and the AoT has 20 heads.}
    \end{subfigure}
    \hfill
    \begin{subfigure}{\linewidth}
        \centering
        \includegraphics[width=\linewidth]{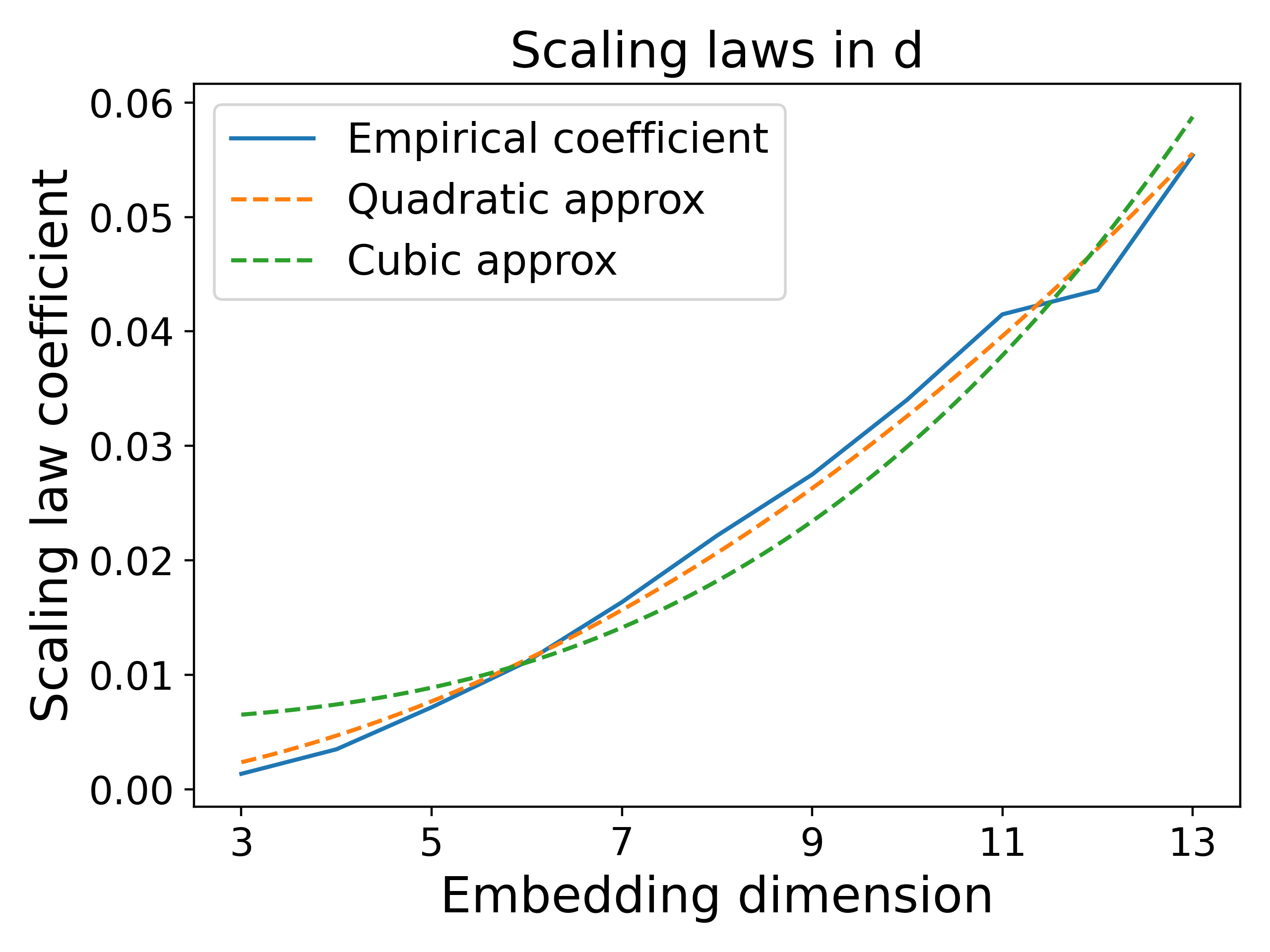}
        \caption{Experiment 4. Scaling laws of the coefficients obtained by linear regression of the scaling laws on $H$ as in figure 1.a. $N=50$}
    \end{subfigure}
    \caption{Two different measurements of the constant $C(d,N)$. On the left when only $d$ varies, and on the right when $d=d_h$ varies.}
\end{figure}

\subsection{MLP memorization}
\label{section:MLP}

The motivation for this paper was to understand the memorization power of the attention mechanism, especially in terms of associative memory. We now know that the attention mechanism can memorize associations, but it is still unknown how memorizing with attention heads compares to memorizing with an MLP. To understand whether attention heads or MLPs will be memorizing in LLMs, we want to compare the accuracy scaling laws of an MLP-based Transformer with that of an AoT with equal number of parameters. To that end, the AoT will be taken as before with $d=d_h$, and the MLP-based Transformers will be one layer of attention with only one head of dimension $d_h=d$ as well, followed by an MLP layer with width $w$. 
\[\mathcal{T}(t) = W_U(W_2\sigma W_1 + I_d)(W_OW_VA(t)+e(t_S)+pos_S)\]
The non-linearity $\sigma$ is chosen to be the GELU function and is applied component-wise to its input vector. The MLP-based model needs at least one attention head in order to mix the tokens together.

Figure 3 shows that the scaling of the AoT and the MLP-based Transformer are different for the same number of parameters. In particular, the MLP-based Transformer does not have a linear scaling law compared to the AoT. This is because the MLPs are harder to optimimzed, and so when the AoT is done learning, the MLPs have not reached there full potential. Thus, under optimization constraints, the AoT is better at memorization. However, in LLMs, this constraint does not hold, and the greater memorization capacity could explain why the MLPs store the facts and associations.

\begin{figure}[t]
    \centering
    \includegraphics[width=\linewidth]{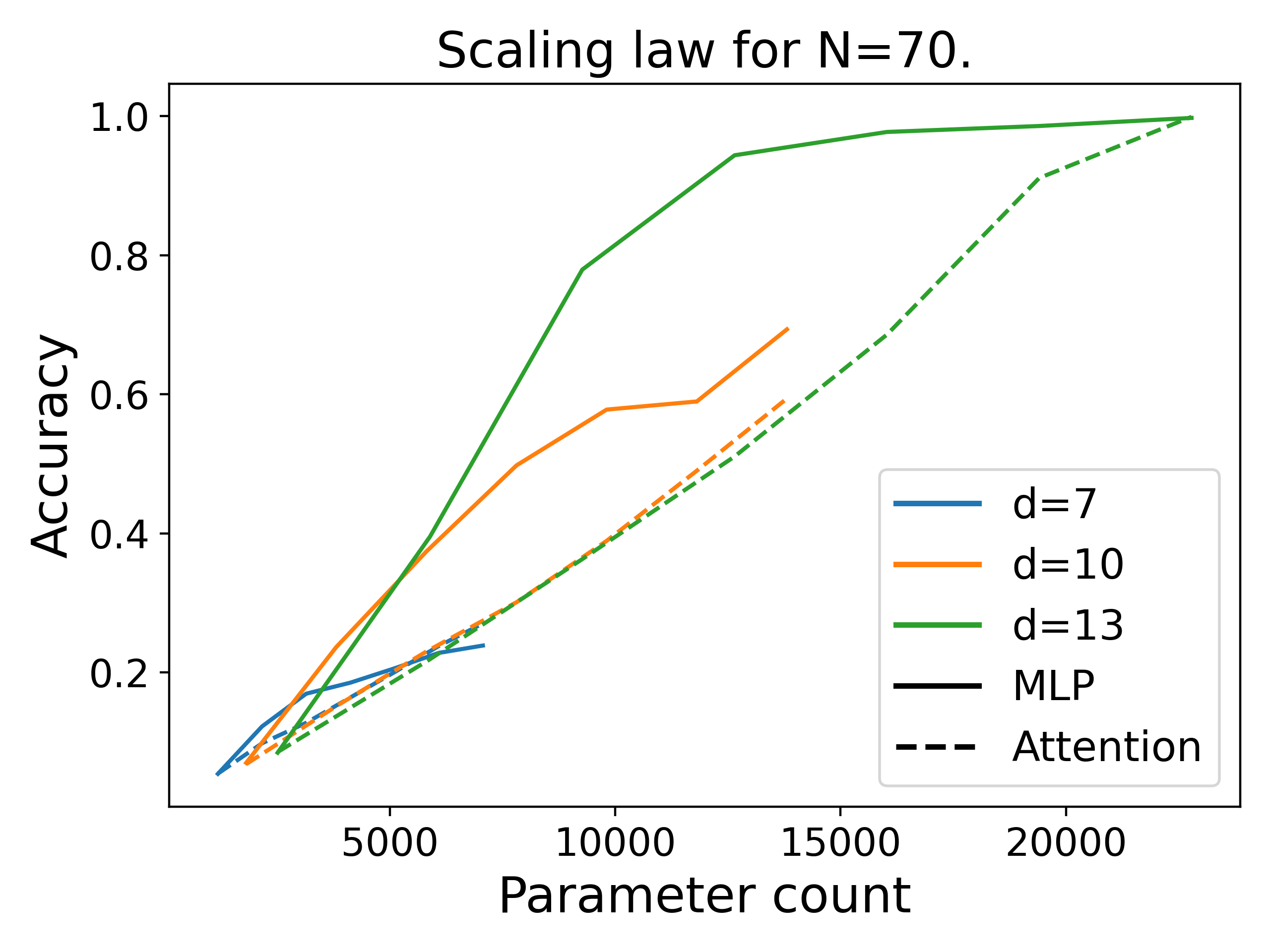}
    \caption{Experiment 5. Accuracy scaling laws on the number of parameters for two models: an AoT and an MLP-based Transformer. Both of them have the same embedding dimension $d$.}
\end{figure}

\section{UPPER BOUND FOR SEQUENCE ENCODERS}
\label{section:bounds}

In Theorem \ref{thm:main}, we are able to upper bound the divergence of the best AoT by the divergence of the best sequence encoder. In this section, we want to upper bound the sequence encoder divergence to have better control on the performance of the best AoT. 

This task is hard in general and is highly dependent on the fact that we are optimizing the KL divergence. For example, when a distribution $\pi$ is close to uniform, it will be easily approximable by a sequence encoder $f$. Especially, the centered logit difference 

\begin{equation*}
    \begin{split}
        \mathcal{Z}_{t_{1:S}}^{t_{S+1}} &:= f(t_{1:S})_{t_{S+1}} - \log(\pi(t_{S+1}|t_{1:S}))\\
        & - \mathbb{E}_{t_{S+1}}[f(t_{1:S})_{t_{S+1}} - \log(\pi(t_{S+1}|t_{1:S}))]
    \end{split}
\end{equation*}

will be close to 0. Doing a Taylor on the KL divergence reveals that the first order term will be the $L_2$ norm of $\mathcal{Z}_{t_{1:S}}^{t_{S+1}}$. In comparison, if we optimize for the total variation loss, one would find a different first order approximation.

In a different setting than the previous remark, we are able to control the divergence when the target probability satisfies Assumption \ref{assump:1}. We can think of $\pi$ as implementing a look-up table with some noise, and we define $g:[N]^S\rightarrow [N]$ the true look-up table meaning that $\pi(g(t_{1:S})|t_{1:S})\simeq 1$. Although Theorem \ref{thm:2} holds for any distribution $\pi$ and function $g$, the upper bound is to be read with the above relation between $\pi$ and $g$ in mind.

\begin{theorem}
    \label{thm:2}
    Let $g:[N]^S \rightarrow [N]$. There exists $W\in\mathbb{R}^{d,N}$ such that \[||WW^T-I_d||_{\infty} \leq C = \sqrt{\frac{32\log(N+1)}{d}}\] We choose $f_{W,E}$ with $E(t_{1:S}) = \lambda({t_{1:S}})W^T_{g({t_{1:S}})}$, and $\lambda({t_{1:S}})$ the solution to the equation below.
    
    \begin{equation*}
        -H(\pi_{t_{1:S}}) = \log\left(\sum_je^{\lambda({t_{1:S}})(W_j-W_{g({t_{1:S}})})^TW_{g({t_{1:S}})}}\right)
    \end{equation*}
    
    We obtain the bound 
    \begin{equation*}
        \begin{split}
            d_{KL}(\pi, \mathcal{L}(N,S,d)) &\leq \mathbb{E}_{t_{1:S}}\Big[(1-\pi(g({t_{1:S}})|{t_{1:S}}))\\
            &\times\left(\frac{1+2C+Cd_{TV}(\tilde{\pi}_{t_{1:S}}, \pi_{\text{unif}})}{1-2C}\right)\\
            &\times\log\left(\frac{N-1}{e^{-H(\pi_{t_{1:S}})}-1}\right)\Bigg]
        \end{split}
    \end{equation*}
    
    Where, if $t_{S+1}\neq g(t_{1:S})$, $\pi_{\text{unif}}(t_{S+1}) = \frac{1}{N-1}$ and $\tilde{\pi}_{t_{1:S}}(t_{S+1}) = \frac{\pi_{t_{1:S}}(t_{S+1})}{1-\pi(g(t_{1:S})|t_{1:S})}$, and $0$ otherwise for both distributions.
\end{theorem}

\begin{table*}[t!]
    \centering
    \begin{tabular}{ |m{2cm}|m{4.5cm}|m{3.5cm}|m{2.5cm}|m{2.5cm}| } 
      \hline
      Paper & Model used & Task & Memorization capacity & Parameter count \\
      \hline
      \citet[Theorem 1]{mahdavi2024memorization} & \rule{0pt}{3ex}One layer, Attention-only Vision Tranformer. Keys and queries may be different vectors. The model doesn’t include a skip connection compared to the other two. & The model maps a sequence of key vectors and one query vector to a desired vector. \newline $\mathcal{T}:\mathbb{R}^{d,S}\times\mathbb{R}^{d}\rightarrow\mathbb{R}^{d}$ & $T_0=H(d_h-1)+1$ & $O(S+N+T_0)$ \\
      \hline
      \citet[Theorem 4.1]{kim2023provable} & \rule{0pt}{3ex}Language Transformer with multiple layers of Attention and MLPs. The attention layers have only one head of inner dimension 1, and the MLPs have width 16. The depth is of size $L$. & \rule{0pt}{3ex}The model maps a sequence of tokens to another sequence of tokens consisting of $S$ times the same token.\newline$\mathcal{T}:[N]^{S}\rightarrow[N]^{S}$ & $T_0=O(L^2)$ & $O(S+N+T_0\log(T_0))$  \\
      \hline
      This paper.\newline(Corollary 1) & \rule{0pt}{3ex}One layer, Attention-only Language Transformer. Keys and queries are the same vectors. & \rule{0pt}{3ex}The model maps an input sequence of tokens to a correct next token.\newline$\mathcal{T}:[N]^{S}\rightarrow[N]$ & $T_0=Hd_h+d$  & $O(S+N+T_0)$ \\
      \hline
    \end{tabular}
    \caption{Summary of the comparison between our AoT's Memorization capacity and previous articles' capacity.}
\end{table*}

One can see that at the limit, the above bound converges to 0 for both the case $H(\pi_{t_{1:S}}) = 0$ meaning that the distribution is indeed a look-up table, or for $H(\pi_{t_{1:S}}) = -\log(N)$ meaning that the distribution is uniform. These cases correspond respectively to $\lambda(t_{1:S}) \rightarrow +\infty$ or $\lambda(t_{1:S}) \rightarrow 0$. A simpler bound for Theorem \ref{thm:2} can be found by simplifying the expectation, and gives 

\begin{equation*}
    \begin{split}
        d_{KL}(\pi, \mathcal{L}&(N,S,d)) \leq \mathbb{E}_{t_{1:S}}\left[\pi(t_{S+1}\neq g({t_{1:S}})|{t_{1:S}})\right]\\
        &\times\log\left(\frac{N-1}{e^{-\mathbb{E}_{t_{1:S}}\left[H\left(\pi_{t_{1:S}}\right)\right]}-1}\right)\frac{1+4C}{1-2C}
    \end{split}
\end{equation*}

which only involves the average entropy and the average probability of the event $t_{S+1}\neq g(t_{1:S})$. The constant $C$ comes from using the Johnson-Linderstrauss Lemma on the canonical basis. It may be improved since we only need to control the greatest positive coefficient of $WW^T-I_d$.

\section{RELATED WORK}
\label{section:related_work}

\textbf{Comparison with previous State-of-the-art}: Our Theorem \ref{thm:main} is an improvement of the Theorem 1 by \citet{mahdavi2024memorization} in the case of language models. Indeed, the paper has the fundamental limitation that $S\leq d$, which is unrealistic for large language models\footnote{In GPT2 $S=1024$ while $d=768$, yet in GPT3 $S=2048$ and $d=12288$. So this limitation depends on the use-case made by the language model, but the tendency in industry is to have $S>d$, with most recent context window in the million tokens for Gemini 1.5 Pro.}. This assumption is in part justified by their paper focusing on Vision Transformers (ViT), where keys and queries can come from different sets of vectors. In that case, this introduces a limitation depending on the rank of the family of query vectors $Q$: $T_0\leq H(\min(Q,d_h)-1)+1$. And as stated in their Proposition 2, this bound is tight in case where the keys are shared by all examples, making the queries the only way to distinguish examples. Our result doesn't have such tightness since we choose a word and a positional embedding that will make the set of keys distinct for each example. 

When the keys and queries are chosen to have full rank, their proof still has to limit to $S\leq d$, yet we don't. This is because they use the result from \citet{pmlr-v119-bhojanapalli20a} stating that $W_{QK}^h$ can be chosen to produce any attention pattern $a$ if $S\leq d$. They then solve the linear system $q^t_{1:S}W_{QK}k_j = \log(a_{i,j})$ for each association. Rather than using the attention pattern's expressivity, our proof makes use of the whole attention layer's expressivity, and we are able to obtain the maximal memory capacity of $Hd_h+d$, which is a strict improvement on their memory capacity of $H(d_h-1)+1$.

\textbf{Small-width large-depth MLP-based Transformers}: Another result on the memorization capacity of Transformer is given by \citet{kim2023provable} where they leverage the depth of Transformers. Their setup and architecture is different from ours, making use of the MLP layers and predicting not only the last token but a complete sequence. Still, their Theorem 4.1 is closest to ours in that they predict the same token at every position, and we can adapt their result to allow a fair comparison. They are able to memorize $T_0$ examples using $O\left(S+N+\sqrt{T_0\log(T_0)}\right)$ parameters\footnote{Looking at Definition 3.1 from \citet{kim2023provable}, one can wonder if the result hides a constraint on the embedding dimension. This is true for their general statement, and using packing number one can obtain a lower bound on $d$. However, in our setting, we work with tokens rather than vectors, and, in that case, their result loses this constraint. This can be seen in their proof by acknowledging that A.1 is trivial using a word embedding, and that this only impacts the constant $R$ in A.6 minimally.}. To do so, they use lots of layers which each have a constant number of parameters. In the parameter counts, $O\left(\sqrt{T_0\log(T_0)}\right)$ parameters come from the MLP layers, which is the majority, and the attention layers, although crucial, need only $O(S)$ parameters\footnote{See their Remark 3.7}.

As a comparison, our model in Corollary \ref{corollary:1} scales as $2S+4N+8T_0$ parameters. So, in terms of scaling, combining MLPs and attention layers with depth is much better.

\textbf{Attention simulating MLPs}: Since the MLP memorization properties are well known, one could use attention heads to memorize by simulating MLPs. \citet{huben2023attentiononlytransformersimplementingmlps} proposes a method to simulate MLP layers using attention heads. They prove that an attention head can simulate an MLP neuron if it can use an appropriate attention mask. If the head has to learn the mask, then the construction still uses one attention per neuron, but each head needs to satisfy $d_h\geq S+1$ and $d\geq S+2$. This makes this architecture memorize exactly $T_0$ associations using $O((d+S)(S+N+T_0))$ parameters.

If exact memorization requires strictly more parameters for the same performances than our architecture, combining this idea with the scaling of Theorem 2 by \citet{Bubeck2020NetworkSA} for the approximate memorization in MLPs means that the number of parameters could be competitive using \[O\left(\left(\log\left(\frac{1}{\varepsilon}\right)\frac{T_{\varepsilon}\log(T_{\varepsilon})}{\log\left(\frac{d}{\log(T_{\varepsilon})}\right)}+d(S+N)\right)\left(1+\frac{S}{d}\right)\right)\] to have an $\varepsilon$-approximation. Theorem 1's parameter count is $O(d(S+N+T_{\varepsilon}-d))$, so using attention to simulate MLPs is competitive when \[S\log(N)\log\left(\frac{1}{\varepsilon}\right)=O(d\log(d)).\]

\section{CONCLUSION}
\label{section:conclusion}

In this paper, we have advanced the state of the art in exact memorization by addressing and overcoming the limitation that the sequence size must be smaller than the embedding dimension. Through empirical analysis, we computed scaling laws based on our Corollary \ref{corollary:2}, confirming that the scaling with respect to $H$ is linear. However, for $d_h$, the scaling appears more complex, showing a dependence on $d_h^2$, meaning that the best performances per parameter is achieved for $d=d_h$. In that case, we found the scaling to be $C(N)d^3H$. 

In accordance with our initial supposition, that Attention heads would perform worse than MLPs in terms of memorization, we found that, optimization problems appart, this was generically the case. Additionally, we introduced the concept of memorization in distribution and derived an upper bound on the KL-divergence. Further progress in this area will require a deeper understanding of the expressivity of the attention mechanism, a topic we have begun exploring in this work.

\section{Acknowledgments}
This research was supported in part by the French National Research Agency under the France 2030 program, reference ANR-23-PEIA-0003.

\clearpage

\bibliography{refs}

\clearpage

\section*{Checklist}

\begin{enumerate}

 \item For all models and algorithms presented, check if you include:
 \begin{enumerate}
   \item A clear description of the mathematical setting, assumptions, algorithm, and/or model. [Yes]
   \item An analysis of the properties and complexity (time, space, sample size) of any algorithm. [Not Applicable]
   \item (Optional) Anonymized source code, with specification of all dependencies, including external libraries. [No]
 \end{enumerate}

 \item For any theoretical claim, check if you include:
 \begin{enumerate}
   \item Statements of the full set of assumptions of all theoretical results. [Yes]
   \item Complete proofs of all theoretical results. [Yes]
   \item Clear explanations of any assumptions. [Yes]     
 \end{enumerate}

 \item For all figures and tables that present empirical results, check if you include:
 \begin{enumerate}
   \item The code, data, and instructions needed to reproduce the main experimental results (either in the supplemental material or as a URL). [Yes] As an URL, but not during the peer review for anonymity protection.
   \item All the training details (e.g., data splits, hyperparameters, how they were chosen). [Yes]
         \item A clear definition of the specific measure or statistics and error bars (e.g., with respect to the random seed after running experiments multiple times). [Yes]
         \item A description of the computing infrastructure used. (e.g., type of GPUs, internal cluster, or cloud provider). [Yes]
 \end{enumerate}

 \item If you are using existing assets (e.g., code, data, models) or curating/releasing new assets, check if you include:
 \begin{enumerate}
   \item Citations of the creator If your work uses existing assets. [Not Applicable]
   \item The license information of the assets, if applicable. [Not Applicable]
   \item New assets either in the supplemental material or as a URL, if applicable. [Not Applicable]
   \item Information about consent from data providers/curators. [Not Applicable]
   \item Discussion of sensible content if applicable, e.g., personally identifiable information or offensive content. [Not Applicable]
 \end{enumerate}

 \item If you used crowdsourcing or conducted research with human subjects, check if you include:
 \begin{enumerate}
   \item The full text of instructions given to participants and screenshots. [Not Applicable]
   \item Descriptions of potential participant risks, with links to Institutional Review Board (IRB) approvals if applicable. [Not Applicable]
   \item The estimated hourly wage paid to participants and the total amount spent on participant compensation. [Not Applicable]
 \end{enumerate}

 \end{enumerate}

\newpage
\onecolumn
\appendix
\section{PROOFS}
\label{appendix:proofs}

\subsection{Lemma 2}
Before proving Lemma 2, let us state Lemma 3 below. We will prove it in appendix A.5.

\begin{lemma}
    For any $n\in\mathbb{N}^*$, there exists $x_1,...,x_n >0$ such that for any $P\in\mathbb{Z}[X_1, ..., X_n]^*$, $P(x_1, ..., x_n)\neq 0$.
\end{lemma}

Let us prove this result in the case $T_0\geq Hd$, since otherwise, one can add extra examples to prove the result and remove them later. Recall that throughout the proof $d\geq d_h$. Thanks to Lemma 3, there exists $(N+L)d$ positive numbers $(X_i)$ such that for any non-zero polynomial with coefficient in $\mathbb{Z}$ of degree $d$ on $(N+L)d$ variables, then the polynomial is non-zero when evaluated on $X_i$. We use these numbers to construct our word and positional embeddings by assigning each to a coefficient of our embeddings $e$ and $pos$. We note $x(t, j) = e(t_j)+pos_j$, where $t$ denotes a sequence of $S$ tokens among the $N^S$ possible sequences, and we define 
\[f:W \rightarrow \left(\frac{\sum_{j}^Se^{x(t,S)^TWx(t,j)}x(t,j)}{\sum_{j}^Se^{x(t,S)^TWx(t,j)}}\right)_{t\in[T_0]}\in\mathbb{R}^{d,T_0}\] 
Our goal is to choose matrices $W_{QK}^i$ such that $A = [f(W_{QK}^1), ..., f(W_{QK}^H)] \in\mathbb{R}^{Hd, T_0}$ has rank $Hd$. This is equivalent to showing that its rows are free in $\mathbb{R}^{T_0}$. It is in fact sufficient to have the image of $f$ not contained in any hyperplane of $\mathbb{R}^{d, T_0}$, which is the property we will prove. This way, we can iteratively concatenate to $A$ a new $f(W_{QK}^i)$ that will increase the rank by $d$ as long as the rank is lower than $T_0$.

Let $(v_t)_{t\in[T_0]}$ such that 
\begin{equation}
    \begin{split}
        \forall \,W\in\mathbb{R}^{d, d}, \, \text{rank}(W)=d_h,\, \sum_{t\in[T_0]}v_t^Tf(W)_t = 0
    \end{split}
\end{equation}
Let $i_{\max}(t) = \arg\max_i(x(t,i)_1)$ the index of the greatest $x$ on the first dimension. We can rewrite the equality (6) as 
\begin{equation}
    \begin{split}
        0 &= \sum_{t\in[T_0]}\frac{\sum_{i}^Se^{x(t,S)^TW(x(t,i)-x(t,i_{\max}(t)))}v_t^Tx(t,i)}{\sum_{i}^Se^{x(t,S)^TW(x(t,i)-x(t,i_{\max}(t)))}}\\
        &= \sum_{t\in[T_0]}\sum_{i}^Se^{x(t,S)^TW(x(t,i)-x(t,i_{\max}(t)))}v_t^Tx(t,i)\frac{1}{1+\sum_{i\neq i_{max}(t)}^Se^{x(t,S)^TW(x(t,i)-x(t,i_{\max}(t)))}}
    \end{split}
\end{equation}
Let us only consider for the rest of the proof matrices $W$ such that $W_{i,j} = 0$ if $i\neq 1$ or $j \neq 1$, which has rank 1. In this case, we have the exponent $x(t,S)^TW(x(t,i)-x(t,i_{\max}(t)))$ equal to $x(t,S)_1^TW_{1,1}(x(t,i)_1-x(t,i_{\max}(t))_1)$ so by definition of $i_{\max}(t)$, it is negative when $W_{1, 1} = w$ is positive. Thus taking $w \rightarrow +\infty$ makes every exponential but the index $i_{\max}(t)$ converge to 0. For a large enough $w$ we can use the Taylor expansion of $x\rightarrow \frac{1}{1+x}$ around valid on $[0,1]$.
\begin{equation}
    \begin{split}
        \frac{1}{1+\sum_{i\neq i_{\max}(t)}^Se^{x(t,S)_1^Tw(x(t,i)_1-x(t,i_{\max}(t))_1)}} = \sum_{j=0}^{+\infty}\left(-\sum_{i\neq i_{\max}(t)}^Se^{x(t,S)_1^Tw(x(t,i)_1-x(t,i_{\max})_1)}\right)^j
    \end{split}
\end{equation}
Finally, we let $\mathcal{S}_j(t) = \{(i_1, ..., i_j)\in[S]^j | \forall k\in[1,j], i_k \neq i_{\max}(t)\}$ that enumerates all indices for the product of $j$ sums in (8). Developing these products of sums into a sum of products transforms the term into,
\begin{equation}
    \sum_{j=0}^{+\infty}\sum_{i_1, ..., i_j \in \mathcal{S}(t)}(-1)^je^{x(t,S)_1^Tw\sum_{k=1}^{j}x(t, i_k)_1 - x(t, i_{max}(t)_1)}
\end{equation}
Finally, we put (9) back in (7) to have the following equality for all $w$ sufficiently large.
\begin{equation}
    \sum_{t\in[T_0]}\sum_{j=0}^{+\infty}\sum_{i_0=1}^S\sum_{i_1, ..., i_j \in \mathcal{S}(t)}(-1)^jv_t^Tx(t, i_0)e^{x(t,S)_1^Tw\sum_{k=0}^{j}x(t, i_k)_1 - x(t, i_{max}(t)_1)} = 0
\end{equation}
Equality (10) is a functional equality, were the functions are exponential with different exponents. Since the family of exponential function is free, we have that coefficient in front of each different exponent is 0. The next step in the proof is to find a sequence of indices such that the exponent is almost unique to that sequence of indices and to that token sequence.

Let $t\in[T_0]$ and consider any $n\in [2, d+1]$, with the exponent generated by taking $n^i$ times the index $i\neq i_{\max}(t)$ in the sum in (12). The crafted exponent is $x(t,S)_1w\sum_{i\neq i_{\max}(t)}n^i(x(t,i)_1 - x(t,i_{\max}(t))_1)$. Now let $t'$ another token sequence and $c(t)_i$ other coefficient for the number of time that the index $i$ was chosen, such that both exponents are equal
\begin{equation}
    x(t,S)_1\sum_{i\neq i_{\max}(t)}n^i(x(t,i)_1 - x(t,i_{\max}(t))_1) = x(t',S)_1\sum_{i\neq i_{\max}(t')}c(t')_i(x(t',i)_1 - x(t',i_{\max}(t'))_1)
\end{equation}
We will prove that this implies that $t=t'$ and $c(t')_i = n^i$. Since (11) acts as a polynomial equation on the embedding coefficients, we can use the property of the embedding in Lemma 3 to identify the coefficients on each side. First let us start with the positional embeddings. We only have once the coefficients $pos_S^Tpos_i$ on each side, so we can identify each. This means that $i_{\max}(t) = i_{\max}(t')$, by identifying the negative positional embeddings, and that $c(t')_i = n^i$. To identify the token sequences, we can start by looking at $e(t_S)_1*\left(\sum_{i\neq i_{\max}(t)}^Sn^i(pos_{i, 1} - pos_{i_{\max}(t), 1})\right)$ and compare it to its counterpart, this gives us $e(t_S) = e(t'_S)$ meaning $t_S=t'_S$. Then we remove the extra information we already deduced and the equality (11) becomes
\begin{equation}
    \sum_{i\neq i_{\max}(t)}^Sn^ie\left(t_{1:S}\right)_1 = \sum_{i\neq i_{\max}(t)}^Sn^ie\left(t'_i\right)_1
\end{equation} 
We now apply Lemma 3 to (12), so there need to be as many $e(q),\,q\in[N]$ on each side. Since for $q$ the number of $e(q)$ can be written uniquely in base $n$ on each side, it means that $t_{1:S}=q \Longleftrightarrow t'_i=q$. If the sum in (12) is empty on both side, which it can be, then this means that $t_{i} = t_{i_{max(t)}}$ for all $i$, and likewise for $t'$, and since $t_S=t'_S$, we have $t=t'$.

This shows that, for an exponent of the previous form to be equal to another one, both need to have the same number of each index, they must come from the same token sequence, and their number of indices differ by at most one, depending on whether $i_0 = i_{\max}(t)$ or not. So, we can write the coefficient in front of that exponential as \[v_t^T\sum_{i=1}^{N}x(t,i)N_i^n = 0\] where $N_i^n$ is the number of permutation of a sequence containing $n^k - \mathbb{1}_{i=k\cap i_{\max}(t)\neq k}$ times the index $k$ and starting with index $i$. The system writes as \[v_t^T\mathcal{N}X(t)=0\] with $X(t)_i = x(t,i)$ and $\mathcal{N}_{n,i} = N_i^n$. The matrix $\mathcal{N}X(t)$ has non-zero determinant by property of Lemma 3, so the system's only solution is for $v_t = 0$.

\subsection{Theorem 1}
In what follows we work under Assumption 2, and we let $f_{W,E}$ that achieves the minimum distance $d_{KL}(\pi,\mathcal{L}(N,S,d))$, thanks to Lemma 1.

Let us first prove the theorem when $\varepsilon = 0$. Observe that the skip connection contribution can be expressed as a type of attention head. Indeed, for any matrices $W_{OV}\in\mathbb{R}^{d,d}$ with full rank, and we define the new embeddings as $e'(t_s) = W_{OV}^{-1}e(t_s)$ and $pos'_s = W_{OV}^{-1}pos_s$. Now, if we let $W_{QK}^{\lambda} = \lambda I_d$, and the norm of $pos_S$ much larger than the norm of the other embeddings, we obtain in the limit that $x(t,i) = W_{OV}\sum_{s=1}^Sa(t)_sx'(t,i) + o(e^{-\lambda \alpha})$ as in (1), with some $\alpha>0$. This attention head has inner dimension $d$.

Let us choose $H = \left\lceil \frac{T_0-d}{d_h}\right\rceil$ attention modules. 
We have $\mathcal{T}(t_{1:S}) = W_UW_OW_VA'(t_{1:S}) + o(e^{-\lambda \alpha})$ where $A'$ is the attention before output augmented by the "fake" head from the skip connection. We can now apply Lemma 2 to $A'$, for some choice of $e$, $pos$ and $W_{QK}^h$ such that $A'$ has rank $T_0$. After multiplying by $W_V$ with full rank, the rank stays $T_0$. Finally, we take $W_U = W$, and $W_O$ that solves the system of $T_0$ equations $E = W_OW_VA'$. We end-up with $\mathcal{T} =f_{W,E}+o(e^{-\lambda \alpha})$.

Now, let $\varepsilon>0$, and $T_{\varepsilon}$ the smallest number of sentences whose cumulative probability is greater than $1-\varepsilon$. We have,
\begin{equation}
    \begin{split}
        \left|d_{KL}(\pi, f_{W, E}) - d_{KL}(\pi, \mathcal{T})\right| &= \left|\mathbb{E}_{t_{1:S}}\left[\log\left(\frac{\mathbb{E}_{t_{S+1}}\left[e^{\mathcal{T}(t_{1:S})_{t_{S+1}} - \mathbb{E}_{t_{S+1}}[\mathcal{T}(t_{1:S})_{t_{S+1}}]}\right]}{\mathbb{E}_{t_{S+1}}\left[e^{f_{W}(t_{1:S})_{t_{S+1}} - \mathbb{E}_{t_{S+1}}[f_{W}(t_{1:S})_{t_{S+1}}]}\right]}\right)\right]\right|\\
        &\leq \mathbb{E}_{t_{1:S}}\left[||\mathcal{T}(t_{1:S}) - f_{W}(t_{1:S}) - \mathbb{1}\mathbb{E}_{t_{S+1}}[\mathcal{T}(t_{1:S})_{t_{S+1}}-f_{W}(t_{1:S})_{t_{S+1}}]||_{\infty}\right]\\
        &\leq \mathbb{E}_{t_{1:S}}\left[||(I_d-\mathbb{1}\Pi_{t_{1:S}}^T)(\mathcal{T}(t_{1:S}) - f_{W, E}(t_{1:S}))||_{\infty}\right]\\
        &\leq \mathbb{E}_{t_{1:S}}\left[||(I_d-\mathbb{1}\Pi_{t_{1:S}}^T)W_U(W_OW_VA' - E)t_{1:S}||_{\infty}\right]\\
        &\leq \mathbb{E}_{t_{1:S}}\left[||I_d-\mathbb{1}\Pi_{t_{1:S}}^T||_{2, \infty}||W_U(W_OW_VA' - E)t_{1:S}||_{2}\right]\\
        &\leq 2\mathbb{E}_{t_{1:S}}\left[||W_U(W_OW_VA' - E)t_{1:S}||_{2}\right] + o(e^{-\lambda\alpha})\\
    \end{split}
\end{equation}
Let $P_1^TP_1$ the projection onto the $T_{\varepsilon}$ most likely sentences and $S_1$ the set of those sentences, meaning that for $t_{1:S}\in S_1$, $P_1^TP_1t_{1:S} = t_{1:S}$. Let $P_2^TP_2$ and $S_2$ for the rest of the sentences. We have $I_d = P_1^TP_1 + P_2^TP_2$. Using the singular value decomposition of $W_VA' = UI_{\Sigma}P_1V$, we take $W_O = EP_1^T(P_1VP_1^T)^{-1}I_{\Sigma}^{-1}U^T$\footnote[8]{Here, $P_1VP_1^T$ is invertible since $V$ is orthogonal.}, and we have that $(W_OW_VA' - E)P_1^T = 0$. We also take $W_U = W$ as before.
\begin{equation}
    \begin{split}
        \mathbb{E}_{t_{1:S}}\left[||W(W_OW_VA' - E)t_{1:S}||_{2}\right] &= \mathbb{E}_{t_{1:S}}\left[||W(W_OW_VA' - E)P_2^TP_2t_{1:S}||_{2}\right]\\
        &= \mathbb{E}_{t_{1:S}}\left[||WE(P_1^T(P_1VP_1^T)^{-1}P_1V - I_d)P_2^TP_2t_{1:S}||_{2}\right]\\
        &\leq ||WE||_2\mathbb{E}_{t_{1:S}}\left[||(P_1^T(P_1VP_1^T)^{-1}P_1V - I_d)P_2^TP_2t_{1:S}||_{2}\right]\\
        &\leq \varepsilon||WE||_2\sqrt{1+||(P_1VP_1^T)^{-1}P_1VP_2^T||_2^2}\\
    \end{split}
\end{equation} 
Thus, there exists matrices $e$, $pos$, $W_O^p$, $W_V^p$, and $W_{QK}^p$ such that \[\left|d_{KL}(\pi, f_{W,E}) - d_{KL}(\pi, \mathcal{T})\right| \leq C\varepsilon||WE||_2 + o(e^{-\lambda\alpha})\]

\subsection{Theorem 2}
Recall from the theorem that $f(t_{1:S}) = f_{W,E}(t_{1:S}) = \lambda(t_{1:S})WW_{g(t_{1:S})}^T$, so by definition of $C$, one has $f(t_{1:S})_{j}\leq C$ if $j\neq g(t_{1:S})$, and $f(t_{1:S})_{g(t_{1:S})}\leq 1+C$.
\begin{equation*}
    \begin{split}
        d_{KL}(\pi, f) & = \mathbb{E}_{t_{1:S}}[H(\pi_{t_{1:S}})] - \mathbb{E}_{{t_{1:S}},{t_{S+1}}}\left[\log\left(\frac{f_{t_{S+1}}({t_{1:S}})}{\sum_je^{f_j({t_{1:S}})}}\right)\right]\\
        &= \mathbb{E}_{t_{1:S}}[H(\pi_{t_{1:S}})] + \mathbb{E}_{{t_{1:S}},{t_{S+1}}}\left[\log\left(\sum_je^{f_j({t_{1:S}})-f_{t_{S+1}}({t_{1:S}})}\right)\right]\\
        &= \mathbb{E}_{t_{1:S}}[H(\pi_{t_{1:S}})] + \mathbb{E}_{{t_{1:S}}}\left[\log\left(\sum_je^{f_j({t_{1:S}})-f_{g({t_{1:S}})}({t_{1:S}})}\right)\right] \\
        &+ \mathbb{E}_{{t_{1:S}},{t_{S+1}}}[f_{g({t_{1:S}})}({t_{1:S}})-f_{t_{S+1}}({t_{1:S}})]\\
        &= \mathbb{E}_{t_{1:S}}[H(\pi_{t_{1:S}})] + \mathbb{E}_{{t_{1:S}}}\left[\log\left(\sum_je^{\lambda({t_{1:S}})(W_j-W_{g({t_{1:S}})})^TW_{g({t_{1:S}})}}\right)\right] \\
        &+ \mathbb{E}_{{t_{1:S}},{t_{S+1}}}[f_{g({t_{1:S}})}({t_{1:S}})-f_{t_{S+1}}({t_{1:S}})]
    \end{split}
\end{equation*}
For each ${t_{1:S}}$, $H(\pi_{t_{1:S}})\in [0,\log(N)]$ and 
\begin{equation*}
    \begin{split}
        \lambda \rightarrow \log\left(\sum_je^{\lambda({t_{1:S}})(W_j-W_{g({t_{1:S}})})^TW_{g({t_{1:S}})}}\right)
    \end{split}
\end{equation*}
is decreasing from $\log(N)$ to $0$. Thus, there exists a solution $\lambda({t_{1:S}})$ to equation (6). Moreover since $W_{t_{S+1}}W_{t_{S+1}'}^T \leq C$ for $t_{S+1}'\neq {t_{S+1}}$ and $||W_{g({t_{1:S}})}||^2_2\geq 1-C$,
we obtain the following bound on $\lambda$, \[\lambda({t_{1:S}}) \leq \frac{1}{1-2C}\log\left(\frac{N-1}{e^{-H(\pi_{t_{1:S}})}-1}\right)\]
Taking $\lambda({t_{1:S}})$ to be this solution leaves us with,
\begin{equation*}
    \begin{split}
        d_{KL}(\pi, f) &= \mathbb{E}_{{t_{1:S}},{t_{S+1}}}[f({t_{1:S}})_{g({t_{1:S}})}-f({t_{1:S}})_{t_{S+1}}]\\
        &= \sum_{{t_{1:S}}}\pi({t_{1:S}})f({t_{1:S}})_{g({t_{1:S}})} - \sum_{{t_{1:S}},{t_{S+1}}}\pi({t_{1:S}})\pi({t_{S+1}}|{t_{1:S}})f({t_{1:S}})_{t_{S+1}}\\
        &= \sum_{{t_{1:S}}}\pi({t_{1:S}})(1-\pi(g({t_{1:S}})|{t_{1:S}}))f({t_{1:S}})_{g({t_{1:S}})} - \sum_{{t_{S+1}}\neq g({t_{1:S}})}\pi({t_{1:S}})\pi({t_{S+1}}|{t_{1:S}})f({t_{1:S}})_{t_{S+1}}\\
        &= \sum_{{t_{1:S}}}\pi({t_{1:S}})(1-\pi(g({t_{1:S}})|{t_{1:S}}))\left(f({t_{1:S}})_{g({t_{1:S}})}+\frac{1}{N-1}\sum_{{t_{S+1}}}f({t_{1:S}})_{t_{S+1}}\right)\\
        &- \sum_{{t_{1:S}},{t_{S+1}}}\pi({t_{1:S}})\left(\pi({t_{S+1}}|{t_{1:S}})-\frac{1}{N-1}(1-\pi(g({t_{1:S}})|{t_{1:S}}))\right)f({t_{1:S}})_{t_{S+1}}\\
        &\leq \sum_{{t_{1:S}}}\pi({t_{1:S}})(1-\pi(g({t_{1:S}})|{t_{1:S}}))\lambda({t_{1:S}})\left(1+2C+Cd_{TV}(\tilde{\pi}_{t_{1:S}}, \pi_{\text{unif}})\right)\\
        &\leq \mathbb{E}_{t_{1:S}}\left[(1-\pi(g({t_{1:S}})|{t_{1:S}}))\log\left(\frac{N-1}{e^{-H(\pi_{t_{1:S}})}-1}\right)\left(\frac{1+2C+Cd_{TV}(\tilde{\pi}_{t_{1:S}}, \pi_{\text{unif}})}{1-2C}\right)\right]
    \end{split}
\end{equation*}

\subsection{Theorem 3}
\label{section:poly_approx}

In this appendix, our goal is to formalize the intuition from section 4.1 on how to get rid of the assumption on $\pi$ in Theorem 1. The main problem with no assumption on $\pi$ is that the weights of the sequence encoder that best fit $\pi$ goes to infinity. The rate at which it goes to infinity is important to control the term $\varepsilon||W_{\varepsilon}E_{\varepsilon}||_2$ in Theorem 1. Thus, we need to approximate $\pi$ by another distribution, which we know we can approximate slowly.

\begin{defi}
    A distribution $\pi$ is called \textbf{polynomialy approximable} if there exist a sequence of sequence encoders $f_{\delta}$ and constants $c>0, \alpha\geq 1$ such that for every token sequence $t$ such that $\pi(t_{S+1}|t_{1:S})>0$, we have
    \begin{equation}
        c\delta^{\alpha} \leq \left|\log\left(\frac{\pi(t_{S+1}|t_{1:S})}{\text{softmax}(f_{\delta}(t_{1:S}))_{t_{S+1}}}\right)\right| \leq \delta
    \end{equation}
    We denote $\mathcal{M}(N,S,d)$ the set of all slowly approximable distribution, and \[d_{KL}(\pi, \mathcal{M}(N,S,d)) := \inf_{\phi\in\mathcal{M}(N,S,d)}d_{KL}(\pi,\phi)\]
\end{defi}

With equation (15), we can control nicely the growth of the norm of the sequence of sequence encoder $f_{\delta}$ that approximate the distribution, using the following Lemma.

\begin{lemma}
    Let $\pi$, $c>0$ and $f_{\delta}$ that satisfy equation (15) when $\pi(t_{S+1}|t_{1:S})>0$. Then, for all $t_{1:S}$, $\max_i f_{\delta}(t_{1:S})_i - \min_i f_{\delta}(t_{1:S})_i = O\left(\log\left(\frac{1}{\delta}\right)\right)$.
\end{lemma}

Lemma 4 states that as long as the sequence encoders approximate the distribution at a speed at most polynomial in $\delta$, then the norm will not be too large. If the speed in the lower bound was $\delta^{\frac{1}{\delta}}$, the resulting norm would be $O\left(\frac{1}{\delta}\log\left(\frac{1}{\delta}\right)\right)$, which would diverge too fast for Theorem 1. In equation (15), one can always suppose that the right hand-side is satisfied by re-indexing the sequence $f_{\delta}$. \textbf{Whether any distribution which is the limit of a sequence encoder is in the closure of \boldmath $\mathcal{M}(N,S,d)$ \unboldmath is an open question}. Still, we can provide an upper bound like Theorem 1, but this time for all $\pi$, using $\mathcal{M}(N,S,d)$ instead of $\mathcal{L}(N,S,d)$. 

\begin{theorem}
    Let $\varepsilon > 0$. There exist an AoT $\mathcal{T}^*$ with embedding dimension $d$, head dimension $d_h$, and $H$ attention heads, satisfying $d_hH+d \geq T_{\epsilon}$, such that 
    \begin{equation}
        d_{KL}(\pi, \mathcal{T}^*) \leq d_{KL}(\pi, \mathcal{M}(N,S,d)) + O\left(\varepsilon\log\left(\frac{1}{\varepsilon}\right)\right)
    \end{equation}
    $\mathcal{T}^*$ has $d(S+2N+4d_hH)$ parameters.
\end{theorem}

\begin{proof}\textit{Theorem 3.}
    Let $\pi_0$ a distribution that is slowly approximable and $f_{\delta}$ a sequence of sequence encoders satisfying equation (15) for $c>0$ and $\alpha\geq 1$. First, using the right hand-side of (15), we have 
    \begin{equation}
        \begin{split}
            |d_{KL}(\pi, f_{\delta}) - d_{KL}(\pi, \pi_0)| &\leq \left|\mathbb{E}_{t_{1:S}, t_{S+1}}\left[\log\left(\frac{\pi_0(t_{S+1}|t_{1:S})}{\text{softmax}(f_{\delta}(t_{1:S}))_{t_{S+1}}}\right)\right]\right|\\
            & \leq \mathbb{E}_{t_{1:S}}\left[\sum_{t_{S+1}=1}^N\pi(t_{S+1}|t_{1:S})\left|\log\left(\frac{\pi_0(t_{S+1}|t_{1:S})}{\text{softmax}(f_{\delta}(t_{1:S}))_{i}}\right)\right|\right]\\
            &\leq \delta
        \end{split}
    \end{equation}
    since we apply (15) if $\pi(i|t_{1:S})>0$, and the bound is otherwise trivial. Now, we use the bound from Theorem 1 in equation (13) and (14), but we don't upper bound $||I_d - \mathbb{1}\Pi^T_{t_{1:S}}||_{2, \infty}$. 
    \begin{equation}
        \begin{split}
            |d_{KL}(\pi, \mathcal{T}) - d_{KL}(\pi, f_{\delta})| &\leq \varepsilon C \max_{t_{1:S}}||(I_d - \mathbb{1}\Pi^T_{t_{1:S}})f_{\delta}||_{\infty}\\
            &\leq \varepsilon C \max_{t_{1:S}} \max_{||x||_{\infty}=1} 
            \big\{
            \begin{split}
                \max_{t_{S+1}} f_{\delta}(x)_{t_{S+1}} - \mathbb{E}_{t_{S+1}}[f_{\delta}(x)_{t_{S+1}}]\\
                \mathbb{E}_{t_{S+1}}[f_{\delta}(x)_{t_{S+1}}] - \min_{t_{S+1}} f_{\delta}(x)_{t_{S+1}}
            \end{split}\\
            &\leq \varepsilon C \max_{||x||_{\infty}=1}(\max_{t_{S+1}} f_{\delta}(x)_{t_{S+1}} - \min_{t_{S+1}} f_{\delta}(x)_{t_{S+1}})\\
            &\leq \varepsilon C \max_{t_{1:S}}(\max_{t_{S+1}} f_{\delta}(t_{1:S})_{t_{S+1}} - \min_{t_{S+1}} f_{\delta}(t_{1:S})_{t_{S+1}})\\
            &\leq O\left(\varepsilon\log\left(\frac{1}{\delta}\right)\right)
        \end{split}
    \end{equation}
    by using Lemma 4 for the last row. We conclude the proof by assembling (17) and (18) with $\varepsilon = \delta$, and by taking the infimum over all polynomialy approximable distributions.
    \begin{equation}
        \begin{split}
            d_{KL}(\pi, \mathcal{T}^*) &= d_{KL}(\pi, \mathcal{T}^*) - d_{KL}(\pi, f_{\delta}) + d_{KL}(\pi, f_{\delta}) - d_{KL}(\pi, \pi_0)+d_{KL}(\pi, \pi_0)\\
            &\leq O\left(\varepsilon\log\left(\frac{1}{\varepsilon}\right)\right)+\varepsilon+d_{KL}(\pi, \pi_0)\\
            &\leq O\left(\varepsilon\log\left(\frac{1}{\varepsilon}\right)\right)+d_{KL}(\pi, \mathcal{M}(N,S,d))\\
        \end{split}
    \end{equation}
\end{proof}

\begin{proof}
    \textit{Lemma 4.} Let $\pi$, $c>0$, $\alpha\geq 1$ and $f_{\delta}$ satisfying (15). Let $t_{1:S}$ and $t_{S+1}$ such that $\pi(t_{S+1}|t_{1:S})>0$. We can write the left hand-side of (15) as:
    \begin{equation}
        \frac{e^{f_{\delta}(t_{1:S})_{t_{S+1}}}}{\sum_je^{f_{\delta}(t_{1:S})_{j}}} \leq \pi(t_{S+1}|t_{1:S})e^{-c\delta^{\alpha}}\text{ or }\frac{e^{f_{\delta}(t_{1:S})_{t_{S+1}}}}{\sum_je^{f_{\delta}(t_{1:S})_{j}}} \geq \pi(t_{S+1}|t_{1:S})e^{c\delta^{\alpha}}
    \end{equation}
    Suppose that all choice of $t_{S+1}$ end up in the right hand-side of (20), then, by summing this inequality for all $t_{S+1}$ such that $\pi(t_{S+1}|t_{1:S})>0$, we get \[1\geq \sum_{k}\frac{e^{f_{\delta}(t_{1:S})_{k}}}{\sum_je^{f_{\delta}(t_{1:S})_{j}}} \geq e^{c\delta^{\alpha}} > 1\] which is a contraction. Thus, there exist a token $t_{S+1}$ such that the left hand-side is true. We now keep this token for the rest of the proof.
    \noindent Since $\frac{e^{f_{\delta}(t_{1:S})_{t_{S+1}}}}{\sum_je^{f_{\delta}(t_{1:S})_{j}}} \underset{\delta\rightarrow 0}{\longrightarrow} \pi(t_{S+1}|t_{1:S})>0$, then $e^{f_{\delta}(t_{1:S})_{j}-f_{\delta}(t_{1:S})_{t_{S+1}}} = O(1)$ for any other $j$ satisfying $\pi(j|t_{1:S})>0$. This is because, all the logit that have non-zero limit probability after renormalization must grow at the same speed. It means that $f_{\delta}(t_{1:S})_{t_{S+1}}-f_{\delta}(t_{1:S})_{j} = \log(\pi(t_{S+1}|t_{1:S})) - \log(\pi(j|t_{1:S})) + o(1)$. One can do the same reasoning on the logits $f_{\delta}(t_{1:S}) - \min_jf_{\delta}(t_{1:S})_{j}$, meaning that all logits are positive. In this case, when $\pi(j|t_{1:S})=0$, we have at the limit that $0\leq \frac{f_{\delta}(t_{1:S})_j - \min f_{\delta}(t_{1:S})}{f_{\delta}(t_{1:S})_{t_{S+1}}- \min f_{\delta}(t_{1:S})} < 1$ since the exponentials need in these cases to grow slower than the one of $t_{S+1}$. We denote this limit $\alpha_j$. We can now write again the inequality (20),
    \begin{equation}
        \begin{split}
            \pi(t_{S+1}|t_{1:S})e^{-\frac{\delta}{c}} &\geq \frac{e^{f_{\delta}(t_{1:S})_{t_{S+1}}}}{\sum_je^{f_{\delta}(t_{1:S})_{j}}} \\
            & = \frac{\pi(t_{S+1}|t_{1:S})e^{o(1)}}{\sum_{\pi(j|t_{1:S})>0}\pi(j|t_{1:S})e^{o(1)} + \sum_{\pi(j|t_{1:S})>0}e^{(\alpha_j-1)(f_{\delta}(t_{1:S})_{t_{S+1}} - \min f_{\delta}(t_{1:S}))(1+o(1))}}\\
            &\geq \frac{\pi(t_{S+1}|t_{1:S}) + o(1)}{1 + o(1) + (N-1)e^{(\max_j\alpha_j-1)(\max f_{\delta}(t_{1:S}) - \min f_{\delta}(t_{1:S}))(1+o(1))}}
        \end{split}
    \end{equation}
    And finally, equation (22) for all tokens sequences gives the bound \[\max_{t_{1:S}}(\max f_{\delta}(t_{1:S}) - \min f_{\delta}(t_{1:S})) \leq \frac{\log\left(e^{\frac{\delta}{c}}-1\right) - \log(N-1)}{1-\max_j\alpha_j}(1 + o(1)) = O\left(\log\left(\frac{1}{\delta}\right)\right)\]
\end{proof}
    
\subsection{Other results and proofs}

\begin{proof}
    \label{p-p2}
    \textit{Proposition 3:}
    Since we are under Assumption 1, there exist $g:[N]^S\rightarrow[N]$ such that $\pi(g(t_{1:S})|t_{1:S}) = 1$. Let $W^t_{1:S} = \left(\cos\left(\frac{2\pi i}{N}\right), \sin\left(\frac{2\pi i}{N}\right)\right)$, and $E(t_{1:S}) = W^T_{g(t_{1:S})}$. Let $f_{\lambda} = \lambda WE$. Since $\pi$ has 0 entropy for any of its conditional distribution, we have
    \begin{equation*}
        \begin{split}
            d_{KL}(\pi, f_{\lambda}) &= \mathbb{E}_{t_{1:S}}\left[\log\left(\sum_{j}e^{f_{\lambda}(t_{1:S})_j - f_{\lambda}(t_{1:S})_{g(t_{1:S})}}\right)\right]\\
            &= \mathbb{E}_{t_{1:S}}\left[\log\left(1+\sum_{j\neq g(t_{1:S})}e^{\lambda\cos\left(\frac{2\pi}{N}(j-g(t_{1:S}))\right) - \lambda}\right)\right]\\
            &= \mathbb{E}_{t_{1:S}}\left[\log\left(1+\sum_{j\neq g(t_{1:S})}e^{-\lambda\left(1-\cos\left(\frac{2\pi}{N}(j-g(t_{1:S}))\right)\right)}\right)\right]\\
        \end{split}
    \end{equation*}
    using trigonometric properties. Thus, we have that, $d_{KL}(\pi, f_{\lambda})\underset{\lambda\rightarrow +\infty}{\rightarrow} 0$. 
\end{proof}

\begin{proof}
    \noindent \textit{Lemma 1:}
    Under Assumption 2, for $\varepsilon, \delta>0$, we take $f_{\varepsilon}$ such that \[|d_{KL}(\pi, f_{\varepsilon}) - d_{KL}(\pi, \mathcal{L}(N,S,d))|\leq \varepsilon .\] We can write the new bound \[|d_{KL}(\pi, f_{\varepsilon}) - d_{KL}(\pi, f_{\delta})| = | \mathbb{E}_{t_{1:S},t_{S+1}}[\log(\text{softmax}(f_{\varepsilon}(t_{1:S}))_{t_{S+1}}) - \log(\text{softmax}(f_{\delta}(t_{1:S}))_{t_{S+1}})]| \leq \varepsilon + \delta\]
    Since every probability $\pi(t_{1:S}, t_{S+1}) >0$, we have that for every $t_{1:S}, t_{S+1}$, $\text{softmax}(f_{\varepsilon}(t_{1:S}))_{t_{S+1}})$ is bounded away from 0 and 1 for all $\varepsilon$. This implies that  $f_{\varepsilon}$ is bounded after centering it. So we can extract a converging subsequence from that centered sequence. To finish the proof, we simply need to show the closeness of $\mathcal{L}(N,S,d)$, but this is easily implied by the continuity of the constraint $\text{rank}(WE) \leq d$, which defines this set.
\end{proof}

\begin{proof}
    \noindent \textit{Lemma 3:}
    Let $n\in\mathbb{N}^*$. For each $P\in\mathbb{Z}[X_1, ..., X_n]^*$, $Ker(P) = \{x\in\mathbb{R}^n, P(x)=0\}$ has 0 Lebesgue measure. This can be shown by taking $X\sim\mathcal{N}(0, I_d)$ who has density with respect to the Lebesgue measure and seeing that $P(X)$ still has density by standard properties of $\mathrm{iid}$ gaussian sum and product. This means that $P(X)=0$ has probability 0, making $Ker(P)$ of measure 0. Now, since $\mathbb{Z}[X_1, ..., X_n]^*$ is countable, $\bigcup_{P\in\mathbb{Z}[X_1, ..., X_n]^*}Ker(P)$ still has measure 0. Thus, there exist $x\in\bigcap_{P\in\mathbb{Z}[X_1, ..., X_n]^*}\Bar{Ker(P)}$. Now, these numbers can be chosen positive since the proof works the same way with the quadrant $\mathbb{R}_+^n$.
\end{proof}

\section{Experiments}
\label{appendix:experiments}
We give here details on the experiments in section \ref{section:experiments}. The training procedure for every model was the following: we generate a distribution $\pi$ over 3 tokens, uniform over the first 2, and then one next-token was randomly chosen. This way, $\pi$ satisfies Assumption 1. From $\pi$, we generated $2^6$ batches of $2^{10}$ elements each, and trained the model for $2^6$ epochs. We used Adam with default parameters and a learning rate linearly decreasing from $0.1$ to $0.05$. Each model was trained five times on different seeds, and the accuracies were averaged. \href{https://github.com/leodana2000/Transformer_Attentional_Memory}{This repo} contains the code to reproduce the experiments. All experiments are averages of 5 runs, and were done on a single MacBook Air with an M2 chip and 16 Go of memory, taking about $15$ hours total.

\begin{figure}[t]
    \label{exp_6}
    \centering
    \begin{subfigure}{0.55\textwidth}
        \includegraphics[width=\linewidth]{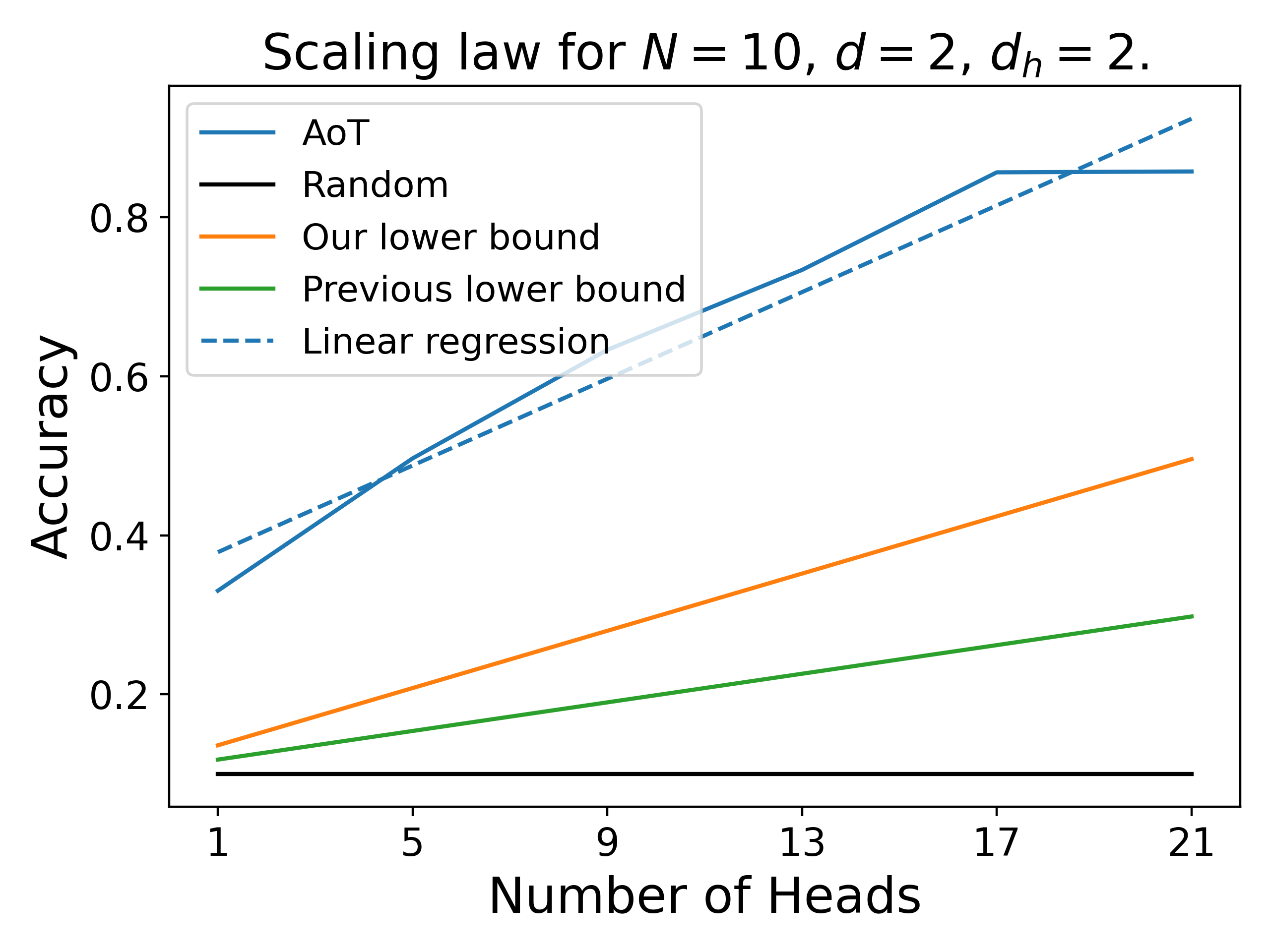}
    \end{subfigure}
    \caption{Experiment 6. Scaling law for $N=10$, with the smallest dimension and head dimension possible $d=d_h=2$. We observe that our lower bound largely underestimates the memorization capacity of a single attention head, but that the scalings are almost identical, with a slope relative difference of $1.3$.}
\end{figure}

In the rest of the appendix, we explain experiment 6 which corresponds to Figure \hyperref[exp_6]{4}, and we present altervatives to experiments 1, 2 and 5 with larger dimensions, or more depth.

As explained earlier, we present here another scaling laws whose intend is to be a fair comparison with our Corollary 2. Exceptionally, we used $N=10$ to avoid training issues. We train an AoT with $d=2$, $d_h=5$ and $H$ from 1 to 20. Corollary 2 states that with $H=20$, the model should be able to obtain exactly 1 accuracy. Now, how far is Corollary 2 from the empirical scaling ? Figure 4 shows that the true scaling at $d=2$ seems to be around $1.3Hd_h$.

\subsection{Larger embedding dimension}

We use $N=200$ dictionnary size and well as $d=50$ to corroborate the experiments in section \ref{section:experiments}. We found that experiments 1 and 5 show similar results. Experiment 2, however, differs and shows a linear scaling of $d_h$ instead of a quadratic as previously found. We did not choose to investigate the reasons behind the finding, but this could be explained by optimization problems in lower dimensions, making it seem like a quadratic scaling. 

\begin{figure}
    \centering
    \includegraphics[width=0.7\linewidth]{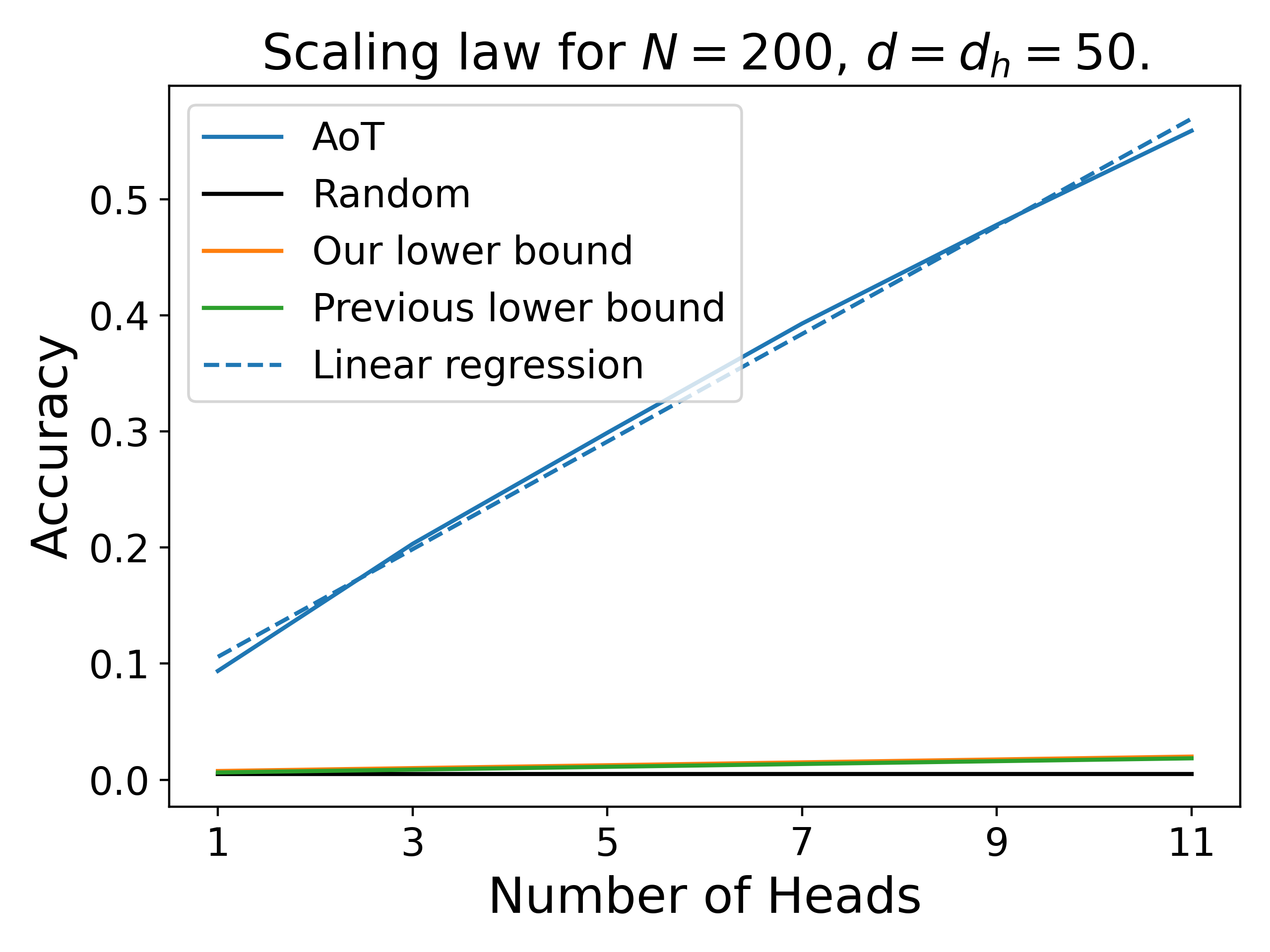}
    \caption{Experiment 1, larger dimension. We plotted the scaling law as in experiment 1, but with $N=200$, and $d=50$. As expected, the scaling remains linear.}
\end{figure}

\begin{figure}
    \centering
    \includegraphics[width=0.7\linewidth]{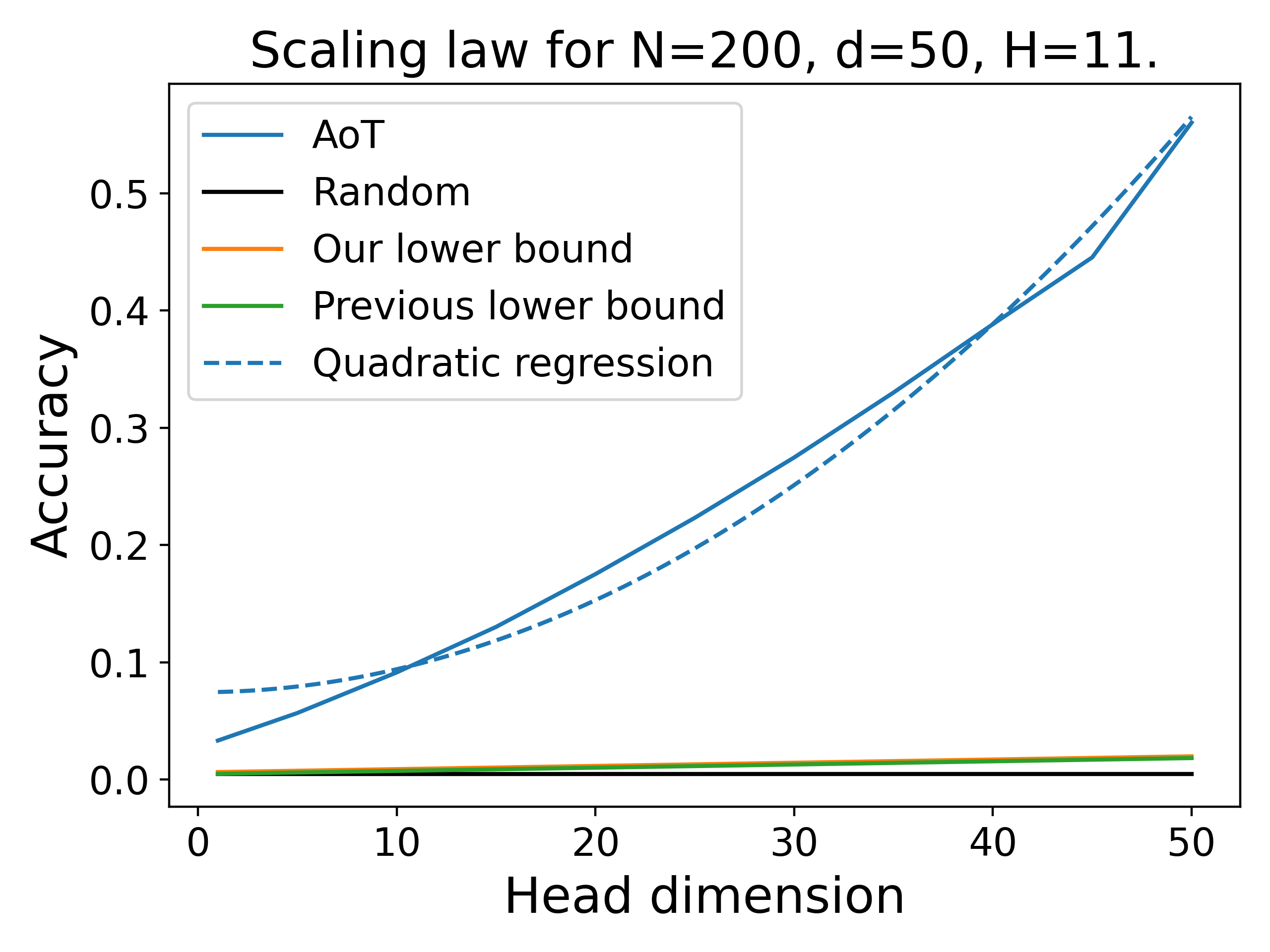}
    \caption{Experiment 2, larger dimension. We plotted the scaling law as in experiment 2, but with $N=200$, and $d=50$. The scaling here is linear. This contradicts the initial scaling laws from experiment 2, making the overall sclaing from experiment 4 quadratic instead of cubic.}
\end{figure}

\begin{figure}
    \centering
    \includegraphics[width=0.7\linewidth]{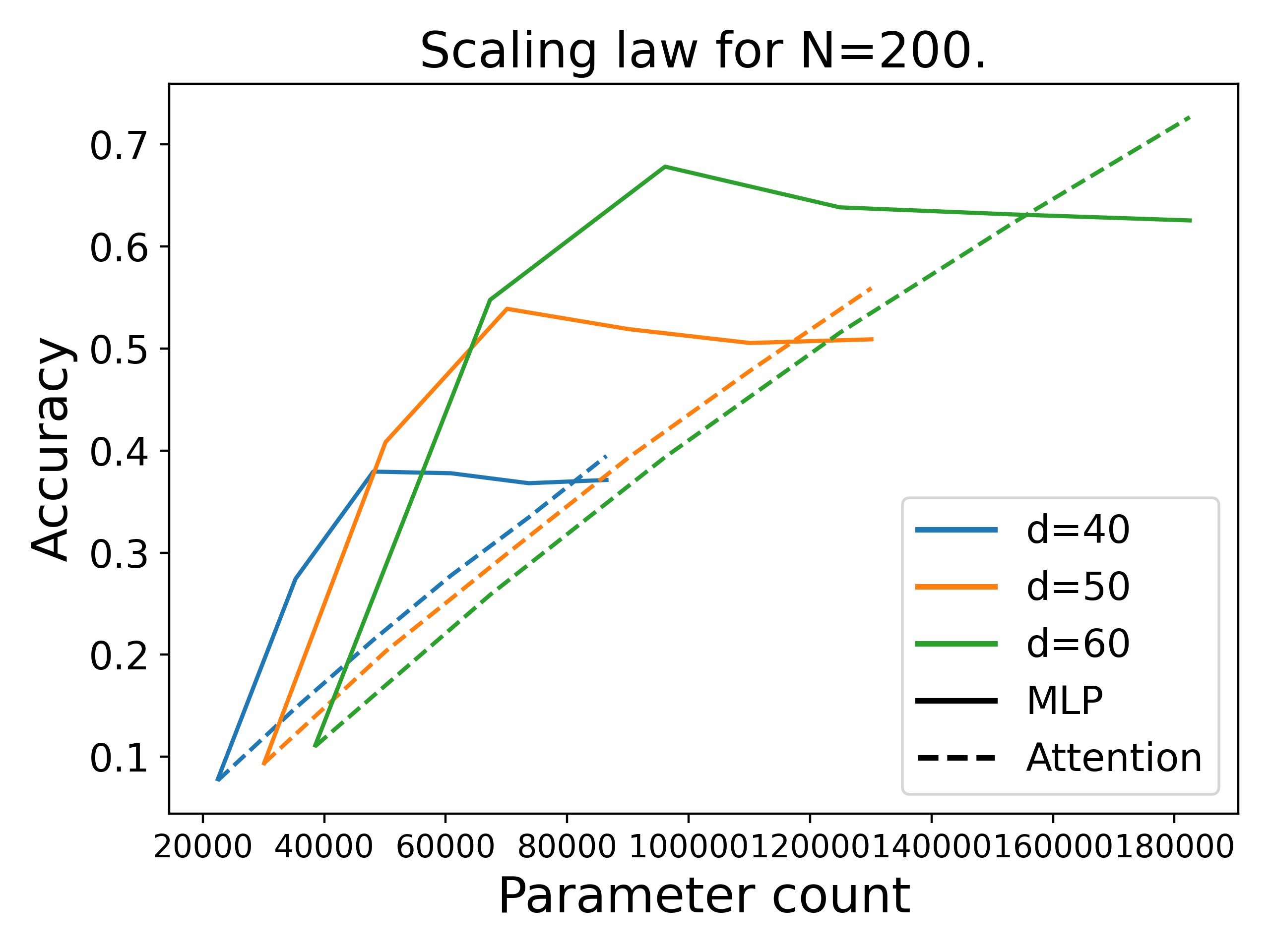}
    \caption{Experiment 5, larger dimension. We plotted the scaling law as in experiment 5, but with $N=200$, and $d=50$. The same phenomenon appears, were the MLPs' memorization capacity is better, yet it is harder to optimize than AoT.}
\end{figure}

\subsection{Deeper Transformers}

We used here AoT and MLP-based Transformer with 5 layers instead of 1, to corroborate the experiments performed in section \ref{section:experiments}. We find similar results for experiments 1, 2 and 5.

\begin{figure}
    \centering
    \includegraphics[width=0.7\linewidth]{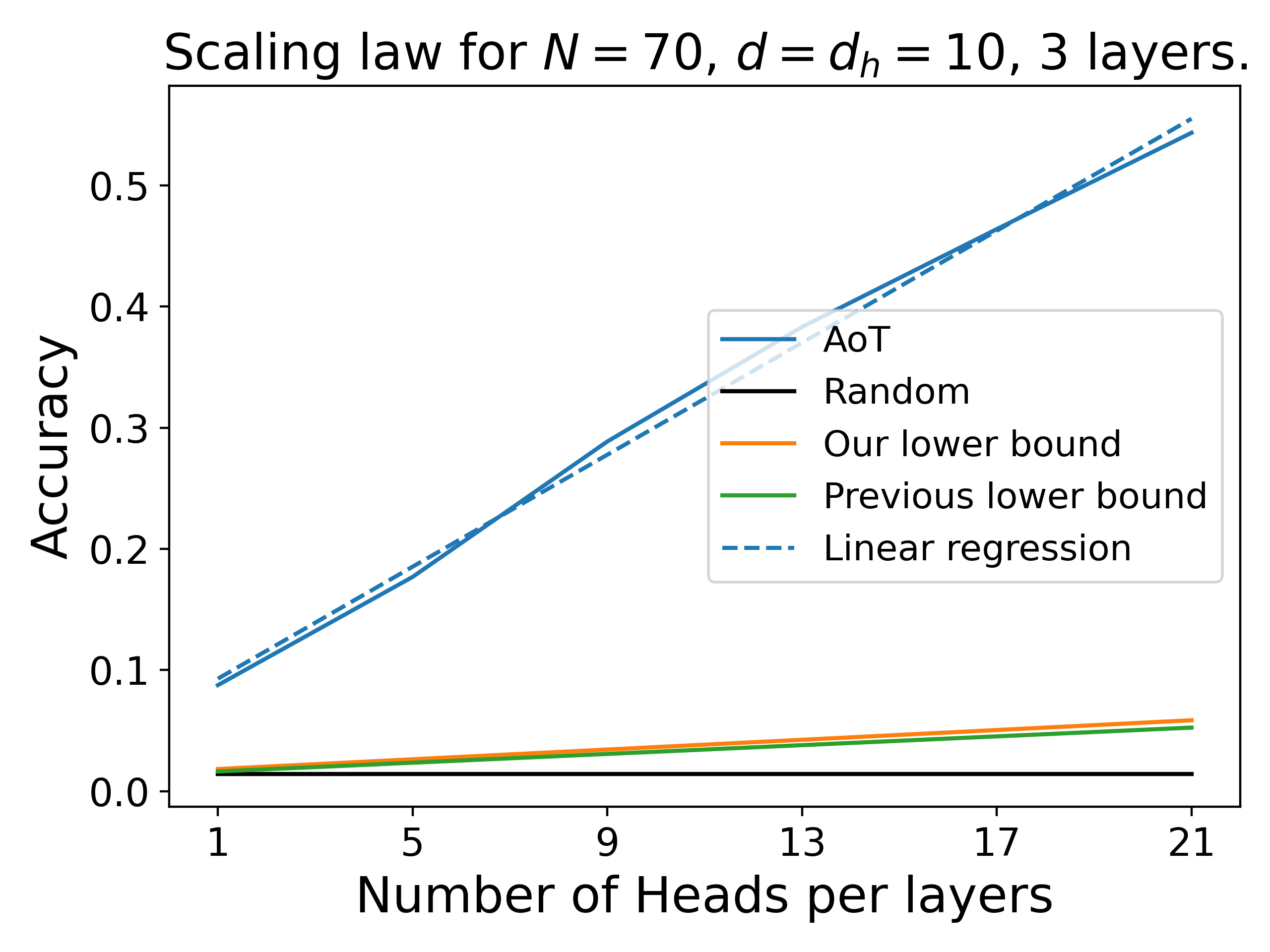}
    \caption{Experiment 1, deeper network. We plotted the scaling law as in experiment 1, but with 5 layers instead. As expected, the scaling is linear.}
\end{figure}

\begin{figure}
    \centering
    \includegraphics[width=0.7\linewidth]{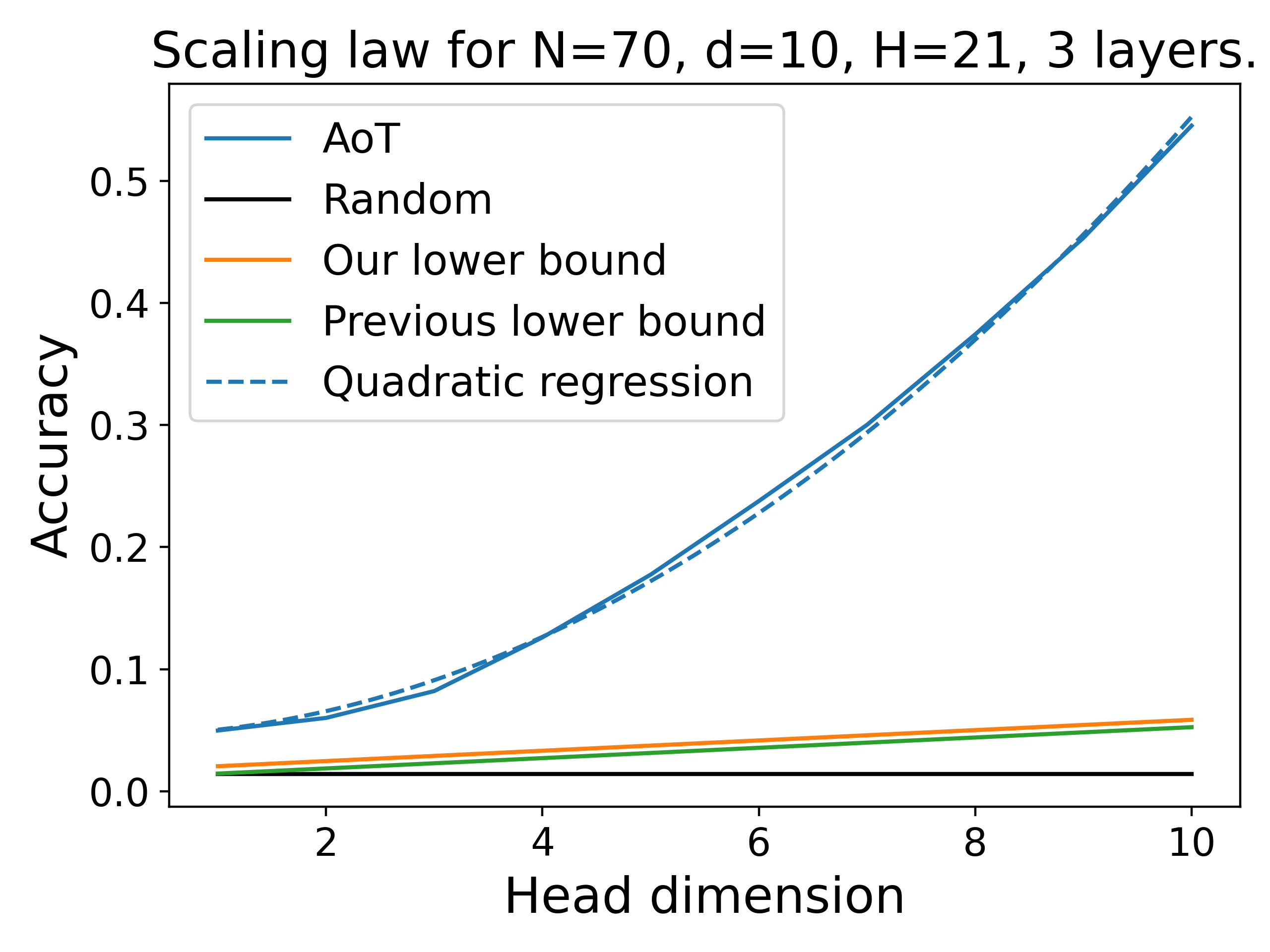}
    \caption{Experiment 2, deeper network. We plotted the scaling law as in experiment 2, but with 5 layers instead. As in experiment 2, the scaling is quadratic.}
\end{figure}

\begin{figure}
    \centering
    \includegraphics[width=0.7\linewidth]{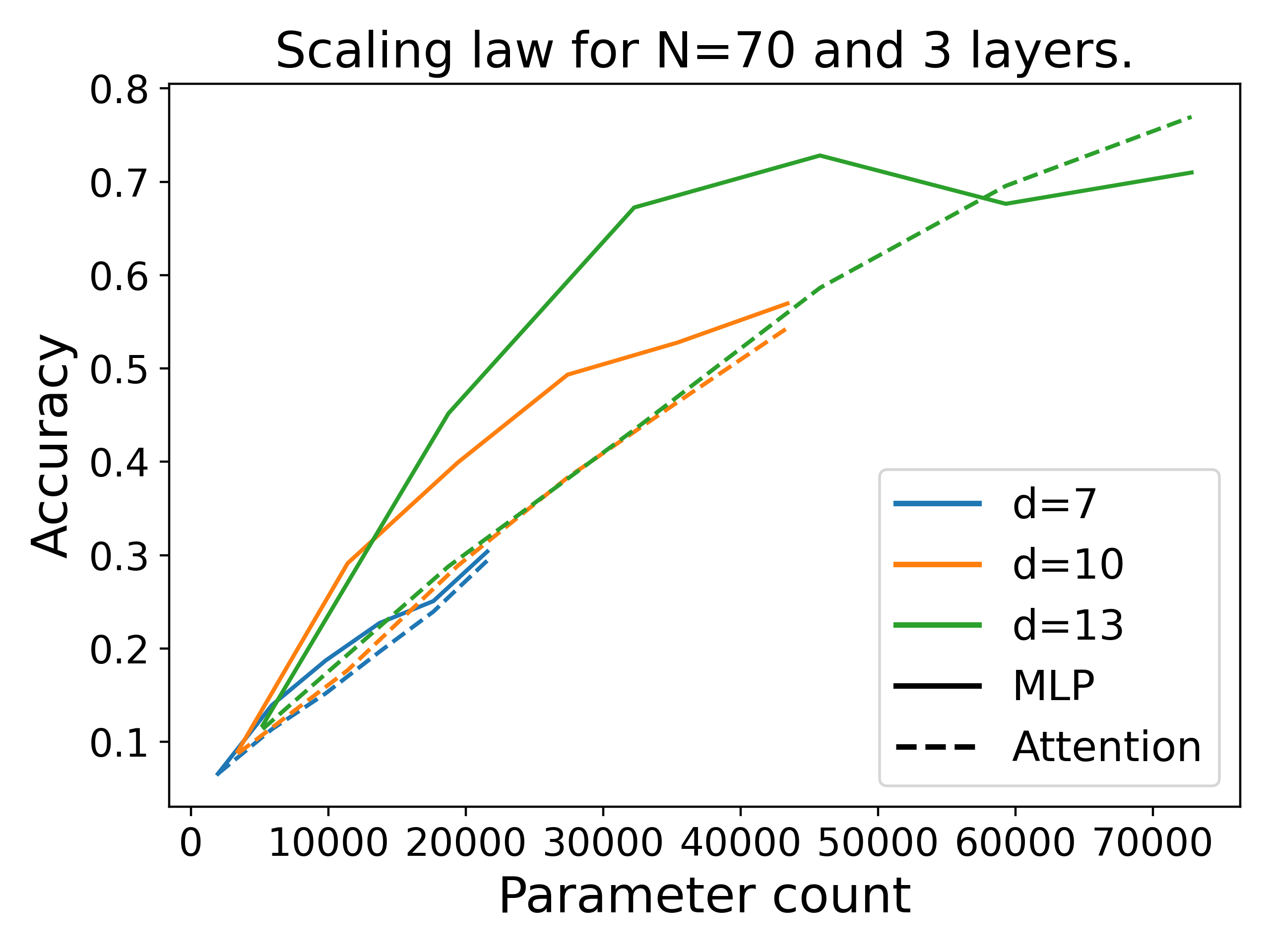}
    \caption{Experiment 5, deeper network. We plotted the scaling law as in experiment 5, but with 5 layers instead. As in experiment 5, the scaling is better but has optimization problem.}
\end{figure}
\end{document}